\documentclass{article}

\usepackage{microtype}
\usepackage{graphicx}
\usepackage{booktabs}
\usepackage{hyperref}

\usepackage{amsmath}
\usepackage{amssymb}
\usepackage{mathtools}
\usepackage{amsthm}
\usepackage{amsfonts}
\usepackage{stackengine}

\usepackage[capitalize,noabbrev]{cleveref}
\usepackage{multicol}
\usepackage{caption}
\usepackage{subcaption}
\usepackage[dvipsnames]{xcolor}
\usepackage{multirow}
\usepackage{bm}
\usepackage[ruled, vlined]{algorithm2e}
\usepackage{appendix}
\usepackage{float}
\usepackage{etoolbox}
\usepackage{wrapfig}
\usepackage{cancel}
\usepackage{times}
\usepackage{etoolbox}

\usepackage[preprint]{corl_2022} 

\everypar{\looseness=-1}

\DeclareMathOperator*{\argmin}{arg\,min}
\newcommand{\barbelow}[1]{\stackunder[1.2pt]{$#1$}{\rule{.8ex}{.075ex}}}

\newcommand{\inner}[2]{\langle#1, #2\rangle}

\newcommand{\Expect}[2]{\mathbb{E}_{#1}\Big[#2\Big]}
\newcommand{\Expectt}[2]{\mathbb{E}_{#1}\Bigg[#2\Bigg]}
\newcommand{\ind}[1]{\mathbb{I}[{#1}]}
\newcommand{\Par}[1]{\Big( #1 \Big)}
\newcommand{\real}{\mathbb{R}}

\newcommand{\diag}[1]{\text{diag}(#1)}

\newcommand{\eye}{\textrm{I}}

\newcommand{\hrulethick}{\specialrule{0.1em}{0em}{0em}}

\newcommand{\Det}[1]{\Big| \text{det}\ #1 \Big|}
\newcommand{\pderiv}[2]{\frac{\partial #1}{\partial #2}}

\newcommand{\Ncal}{\mathcal{N}}

\newcommand{\Ccal}{\mathcal{C}}

\newcommand{\mutilde}{\tilde{\mu}}

\newcommand{\phitilde}{\tilde{\phi}}
\newcommand{\thetatilde}{\tilde{\theta}}

\newcommand{\mutildebold}{\bm{\tilde{\mu}}}

\newcommand{\xhatbold}{\bm{\hat{x}}}
\newcommand{\uhatbold}{\bm{\hat{u}}}

\newcommand{\zhatbold}{\bm{\hat{z}}}
\newcommand{\yhatbold}{\bm{\hat{y}}}
\newcommand{\xhat}{\hat{x}}
\newcommand{\uhat}{\hat{u}}
\newcommand{\zhat}{\hat{z}}

\newcommand{\fhat}{\hat{f}}
\newcommand{\Jhat}{\hat{J}}

\newcommand{\pibold}{\bm{\pi}}
\newcommand{\mubold}{\bm{\mu}}
\newcommand{\thetabold}{\bm{\theta}}

\newcommand{\nfmpc}{\textbf{NFMPC}}

\newcommand{\nfmpco}{\textbf{NFMPC (Obs)}}
\newcommand{\nfmpcno}{\textbf{NFMPC (No Obs)}}

\newcommand{\nfmpcns}{\textbf{NFMPC (No Shift)}}

\newcommand{\mppi}{\textbf{MPPI}}
\newcommand{\flowmppi}{\textbf{FlowMPPI}}
\newcommand{\pngrid}{\textsc{PNGrid}}

\newcommand{\pnrand}{\textsc{PNRand}}
\newcommand{\pnranddyn}{\textsc{PNRandDyn}}
\newcommand{\franka}{\textsc{Franka}}
\newcommand{\frankaobs}{\textsc{FrankaObstacles}}

\definecolor{myblue}{RGB}{119, 177, 233}
\definecolor{mypurple}{RGB}{212, 152, 229}
\definecolor{mygreen}{RGB}{184, 233, 134}
\definecolor{myyellow}{RGB}{233, 213, 134}

\SetKwInput{KwParam}{Parameters}
\SetKwFor{ParallelFor}{for}{do in parallel}{end}

\theoremstyle{plain}
\newtheorem{theorem}{Theorem}[section]

\theoremstyle{definition}

\theoremstyle{remark}

\title{Learning Sampling Distributions for Model Predictive Control}

\author{
  Jacob Sacks \\
  University of Washington \\
  \texttt{jsacks6@cs.washington.edu}
  \And 
  Byron Boots \\
  University of Washington \\
  \texttt{bboots@cs.washington.edu}
}

\begin{document}
\maketitle
\vspace{-4ex}

\begin{abstract}
Sampling-based methods have become a cornerstone of contemporary approaches to Model Predictive Control (MPC), as they make no restrictions on the differentiability of the dynamics or cost function and are straightforward to parallelize.
However, their efficacy is highly dependent on the quality of the \emph{sampling distribution itself}, which is often assumed to be simple, like a Gaussian.
This restriction can result in samples which are far from optimal, leading to poor performance.
Recent work has explored improving the performance of MPC by sampling in a learned latent space of controls.
However, these methods ultimately perform all MPC parameter updates and warm-starting between time steps in the control space.
This requires us to rely on a number of heuristics for generating samples and updating the distribution and may lead to sub-optimal performance.
Instead, we propose to carry out all operations in the latent space, allowing us to take full advantage of the learned distribution.
Specifically, we frame the learning problem as bi-level optimization and show how to train the controller with backpropagation-through-time.
By using a normalizing flow parameterization of the distribution, we can leverage its tractable density to avoid requiring differentiability of the dynamics and cost function.
Finally, we evaluate the proposed approach on simulated robotics tasks and demonstrate its ability to surpass the performance of prior methods and scale better with a reduced number of samples.
\end{abstract}

\keywords{Model Predictive Control, Normalizing Flows}

\vspace{-3ex}
\section{Introduction}
\vspace{-1ex}
Sequential decision making under uncertainty is a fundamental problem in machine learning and robotics.
Recently, model predictive control (MPC) has emerged as a powerful paradigm to tackle such problems on real-world robotic systems.
In particular, MPC has been successfully applied to helicopter aerobatics \citep{abbeel2010autonomous}, robot manipulation \citep{kumar2016optimal, finn2017deep, bhardwaj2021fast}, humanoid robot locomotion \citep{erez2013integrated}, robot-assisted dressing \citep{erickson2018deep}, and aggressive off-road driving \citep{williams2016aggressive, williams2017information, wagener2019online}.
Sampling-based approaches to MPC are becoming particularly popular due to their simplicity and ability to handle non-differentiable dynamics and cost functions.
These methods work by sampling controls from a policy distribution and rolling out an approximate system model using the sampled control sequences.
They then use the resulting trajectories to compute an approximate gradient of the cost function and update the policy.
Next, the controller samples an action from this distribution and applies it to the system.
It repeats the process from the resulting next state, creating a feedback controller.
Between time steps, it warm-starts the optimization process using a modification of the solution at the previous time step.

An important design decision is the form of the sampling distribution, which is often simple, e.g. a Gaussian, such that we can efficiently sample and tractably update its parameters.
However, this also has drawbacks: without much control over the distribution form, samples often lie in high-cost regions, hindering performance.
This can be particularly problematic in complex environments with sparse costs or rewards, as a poorly parameterized distribution may hinder efficient exploration, leading the system into bad local optima.
A side effect is that we often require many samples to accomplish the objective, increasing computational requirements.
There have been extensions which target more complex distributions, such as Gaussian mixture models \citep{okada2020variational} and a particle method based on Stein variational gradient descent (SVGD) \citep{lambert2020stein}.
However, there is a large amount of structure in the environment that these methods fail to exploit.
Instead, an alternative approach is to \emph{learn} a sampling distribution which can leverage the environmental structure to draw more optimal samples.

Prior work on learning MPC sampling distributions generally requires differentiability of the dynamics and cost function \citep{amos2020differentiable, agarwal2020imitative}.
\citet{power2021variational, power2022variational} circumvent this issue by leveraging normalizing flows (NFs) \citep{dinh2014nice, rezende2015variational, dinh2016density}, which have a tractable log-likelihood.
This property allows them to learn flexible distributions by directly optimizing the MPC cost without requiring differentiability via the likelihood-ratio gradient.
However, a limitation of their approach is that all online updates to the distribution and warm-starting between time steps occur entirely in the control space, leaving the latent distribution fixed.
This forces them to apply heuristics to generate samples by combining those from the learned distribution with Gaussian perturbations to the current control-space mean.
These restrictions prevent us from fully taking advantage of the learned distribution and potentially throws away useful information.
Additionally, their approach does not allow for the incorporation of control constraints directly into the sampling distribution, which is important for real-world robots.

Instead, we propose to alter the optimization machinery to operate entirely in the latent space.
As the NF latent space follows a simple distribution, it remains feasible to perform MPC updates in this learned space and update the latent distribution online.
Specifically, during an episode, the parameters of the latent distribution are updated with MPC while those of the NF remain fixed.
Then during training, after each episode, the parameters of the NF are updated.
We can  frame this setup as a bi-level optimization problem \citep{bard2013practical} and derive a method for computing an approximate gradient through the latent MPC update.
This involves treating MPC as a recurrent network, where the control distribution acts as a form of memory, and unrolling the computation to train with backpropagation-through-time (BPTT).
However, it is no longer clear how to warm-start between time steps because there is no clear delineation of time in the latent space.
Moreover, the usual method of warm-starting, which simply shifts the current plan forward in time, may be sub-optimal.
Therefore, we propose to learn a shift model, which performs all warm-starting operations in the latent space.
Finally, we show how to alter the NF architecture to incorporate box constraints on the sampled controls.

\textbf{Contributions:}
In this work, we build upon recent efforts to learn sampling distributions for MPC with NFs by moving all online parameter updates and warm-start operations into the latent space.
We accomplish this by framing the learning problem as bi-level optimization and derive an approximate gradient through the MPC update of the latent distribution to train the network with BPTT.
Additionally, we show how to parameterize the flow architecture such that we can incorporate box constraints on the controls.
Finally, we empirically evaluate our proposed approach on simulated navigation and manipulation tasks.
We demonstrate its ability to improve performance over the baselines by taking full advantage of the learned latent space.
We find that the performance of the controllers with our learned sampling distributions scales more gracefully with a reduction in the number of samples.

\vspace{-1.5ex}
\section{Sampling-Based Model Predictive Control}
\vspace{-1ex}
We consider the problem of controlling a discrete-time stochastic dynamical system, which is in state $x_t \in \real^N$ at time step $t$. 
Upon the application of control $u_t \in \real^M$, the system incurs the instantaneous cost $c(x_t, u_t)$ and transitions to $x_{t+1} \sim f(x_t, u_t)$.
We wish to design a policy $u_t \sim \pi(\cdot | x_t)$ such that the system achieves good performance over $T$ steps.
Instead of finding a single, globally optimal policy, MPC re-optimizes a local policy at each time step by predicting the system's behavior over a finite horizon $H < T$ using an approximate model $\hat{f}$.
Specifically, in sampling-based MPC, we sample a control sequence $\uhatbold_t \sim \pi_{\theta} (\cdot)$, where $\uhatbold_t \triangleq (\uhat_t, \uhat_{t+1}, \cdots, \uhat_{t+H-1})$, and our policy is parameterized by $\theta \in \Theta$. 
We rollout our model starting at $x_t$ using these sampled controls to get our predicted state sequence $\xhatbold_t \triangleq (\xhat_t, \xhat_{t+1}, \cdots, \xhat_{t+H})$, with $\xhat_t = x_t$.
The total trajectory cost is
\begin{equation}
C(\xhatbold_t, \uhatbold_t) = \sum_{h=0}^{H-1} c(\xhat_{t+h}, \uhat_{t+h}) + c_{term}(\xhat_{t+H}),
\end{equation}
where $c_{term}(\cdot)$ is a terminal cost function.
We then construct a statistic $\Jhat(\theta; x_t)$ defined on cost $C(\xhatbold_t, \uhatbold_t)$ and use the rollouts to solve $\theta_t \leftarrow \argmin_{\theta \in \Theta} \Jhat(\theta; x_t)$.
After optimizing our parameters, we sample the control sequence $\uhatbold_t \sim \pi_{\theta_t}(\cdot)$, apply the first control to the real system (i.e. $u_t = \uhat_t$), and repeat the process.
Because each parameter $\theta_t$ depends on the current state, MPC effectively yields a state-feedback policy, even though the individual distributions give us an open-loop sequence.

In this paper, we consider a popular sampling-based MPC algorithm known as Model Predictive Path Integral (MPPI) control \citep{williams2016aggressive, williams2017information}.
MPPI optimizes the exponential utility or risk-seeking objective:
\begin{equation}
\Jhat(\theta; x_t) = - \log \Expect{\pi_\theta, \fhat}{\exp \Par{-\frac{1}{\beta} C(\xhatbold_t, \uhatbold_t)}},
\end{equation}
where $\beta>0$ is a scaling parameter, known as the temperature.
As we do not assume that the dynamics or cost function are differentiable, we compute the gradients via the likelihood-ratio derivative:
\begin{equation}
\nabla \Jhat(\theta; x_t) = - \frac{
    \Expect{\pi_\theta, \fhat}{e^{-\frac{1}{\beta}C(\xhatbold_t, \uhatbold_t)} \nabla_\theta \log \pi_{\theta}(\uhatbold_t)}
}{
    \Expect{\pi_\theta, \fhat}{e^{-\frac{1}{\beta}C(\xhatbold_t, \uhatbold_t)}}
}.
\label{eq:grad_Jhat}
\end{equation}
In MPPI, the policy is assumed to be a factorized Gaussian of the form
\begin{equation}
\pi_{\theta}(\uhatbold) = \prod_{h=0}^{H-1} \pi_{\theta_h}(\uhat_{t+h})
= \prod_{h=0}^{H-1}  \Ncal(\uhat_{t+h}; \mu_{t+h}, \Sigma_{t+h}).
\label{eq:factorized_gaussian}
\end{equation}
Previous work by \citet{wagener2019online} has shown that optimizing this objective with dynamic mirror descent (DMD) \citep{hall2013dynamical} and approximating with Monte Carlo estimates gives us the MPPI update rule:
\begin{equation}
\mu_{t+h} = (1-\gamma_t) \tilde{\mu}_{t+h} + \gamma_t \sum_{i=1}^N w_i \uhat_{t+h}^{(i)},\quad
w_i = \frac{
e^{-\frac{1}{\beta} C(\xhatbold_t^{(i)}, \uhatbold_t^{(i)})}
}{
\sum_{j=1}^N e^{-\frac{1}{\beta} C(\xhatbold_t^{(j)}, \uhatbold_t^{(j)})}
}
\label{eq:MPPI_mean_update_approx}
\end{equation}
where $\mutilde_{t+h}$ is the current mean for each time step and $\gamma_t$ is the step size.
Between time steps of DMD, we get $\mutilde_{t+h}$ from our previous solution $\mu_{t+h}$ by using a shift model $\mutilde_{t+h} = \Phi(\mu_{t+h})$.
This shift model aims to predict the optimal decision at the next time step given the previous solution. 
In the context of MPC, it allows us to warm-start the optimization problem to speed up convergence, as we can only approximately solve the optimization problem due to real-time constraints.

\vspace{-1.5ex}
\section{Learning the Sampling Distribution of MPC}
\vspace{-1ex}
\subsection{Representation of the Learned Distribution}
\vspace{-1ex}
Instead of using uninformed sampling distributions, learned distributions can potentially exploit structure in the environment to draw samples which are more likely to be collision-free and close to optimal.
However, such learned distributions must be sufficiently expressive in order to better capture near-optimal, potentially multimodal, behavior.
They must also be parameterized such that it is tractable to sample from and update online.
If the distribution has a large number of parameters, the number of samples required to efficiently update them online may be computationally infeasible.
And ideally, the form of our distribution would be such that we could find a closed-form update.

One path towards meeting these criteria is to maintain a simple latent distribution from which we can sample, and then learn a transformation of the samples which maps them to a more complex distribution.
During training, we learn the parameters of this transformation, which can be conditioned on problem-specific information, such as the starting and goal configurations of the robot and obstacle placements.
However, when executing the policy during an episode, the parameters of this learned transformation remain fixed, and instead, we update the parameters of the latent distribution.
Concretely, we consider learning a distribution $\pi_{\theta, \lambda}$ which is defined implicitly as follows:
\begin{equation}
\zhatbold_t \sim p_\theta(\cdot), \quad \uhatbold_t = h_\lambda(\zhatbold_t; c)
\end{equation}
where $\zhatbold_t \triangleq (\zhat_t, \zhat_{t+1}, \cdots, \zhat_{t+H-1})$, $c$ is a context variable describing the relevant information of the environment, $p_\theta$ is the latent distribution with parameters $\theta$, and $h_\lambda$ is the learned conditional transformation with parameters $\lambda$.
Moving forward, we assume that both $\zhatbold_t$ and $\uhatbold_t$ are stacked as vectors in $\real^{MH}$.
If $p_\theta$ is a Gaussian factorized as in \Cref{eq:factorized_gaussian} and we assume that $h_\lambda$ is invertible, we prove in \Cref{app:proof_latent_update} that the corresponding DMD update to the latent mean is simply \Cref{eq:MPPI_mean_update_approx}, except that we replace the controls in the weighted sum with the latent samples.

\vspace{-1ex}
\subsection{Formulating the Learning Problem}
\vspace{-0.5ex}

Learning the distribution $\pi_{\theta, \lambda}$ amounts to solving a bi-level optimization problem \citep{bard2013practical}, in which one optimization problem is nested in another.
The lower-level optimization problem involves updating the latent distribution parameters at each time step, $\theta_t$, by minimizing the expected cost with DMD.
The upper-level optimization problem consists of learning $\lambda$ such that MPC performs well across a number of different environments.
To formalize this, first consider that we have some distribution of environments $c \sim \Ccal(\cdot)$ over which we wish MPC to perform well.
For each environment, our system has some conditional initial state distribution $x_0 \sim \rho(\cdot | c)$.
The objective we wish to minimize is then
\begin{equation}
\ell(\thetabold, \lambda; c) = \Expectt{\pibold_{\thetabold,\lambda}, \rho, f}{
\sum_{t=0}^{T-1} \Jhat(\theta_t, \lambda; x_t, c)
}
\label{eq:bilevel-objective}
\end{equation}
where $\thetabold = (\theta_0, \theta_1, \cdots, \theta_{T-1})$ and our cost statistic, $\Jhat$, now depends on $\lambda$ and $c$ as well.
This objective measures the expected performance of the intermediate plans produced by MPC along the $T$ steps of the episode.
Our desired bi-level optimization problem can be formulated as:
\begin{equation}
\begin{aligned}
\min_\lambda\ \Expect{\Ccal}{\ell(\thetabold(\lambda), \lambda; c)} \quad
\text{s.t.}\ \thetabold(\lambda) \approx_\lambda \argmin_{\thetabold} \ell(\thetabold, \lambda; c)
\end{aligned}
\label{eq:bilevel_problem}
\end{equation}
where $\approx_\lambda$ indicates that we approximate the solution of the optimization problem with an iterative algorithm that may also be parameterized by $\lambda$, as the exact minimizer is not available in closed form.
Moreover, the notation $\thetabold(\lambda)$ indicates the dependence of the lower-level solution on the upper-level parameters.
In our case, we solve the lower-level problem with DMD, where we also parameterize the shift model, $\Phi_\lambda(\cdot; c)$, making it a learnable component and conditioned on $c$.

The normal shift model in MPC simply shifts the control sequence forward one time step and appends a zero or random control at the end.
However, because we are performing this update in the latent space, there is no clear delineation between time steps of the latent controls, as they are coupled according to the learned transformation.
Therefore, there is no way to easily perform the equivalent shift operation in the latent space.
As such, we instead learn this shift model along with the transformation.
Besides, the standard approach described above may not be optimal.
By learning it, we may be able to further improve performance.
This is especially true because the performance hinges greatly on the quality of the shift model since we only run one iteration of the DMD update.

\vspace{-1ex}
\subsection{Parameterizing with Normalizing Flows}
\vspace{-0.5ex}
\label{sec:param}
In order to optimize the upper-level objective in \Cref{eq:bilevel_problem} with respect to $\lambda$, we need to be able to compute the density $\pi_{\theta, \lambda}$ directly.
Therefore, we choose to represent $h_\lambda$ with a normalizing flow (NF) \citep{dinh2014nice, dinh2016density, rezende2015variational}, which explicitly learns the density by defining an invertible transformation that maps latent variables to observed data.
Generally, we compose a series of component flows together, i.e. $h_{\lambda} = h_{\lambda_K} \circ h_{\lambda_{K-1}} \circ \dots \circ h_{\lambda_1}$, which define a series of intermediate variables $\yhatbold_0, \dots, \yhatbold_{K-1}, \yhatbold_K$, with $\yhatbold_0=\zhatbold$ and $\yhatbold_K=\uhatbold$.
The log-likelihood of the composed flow is given by:
\begin{equation}
\log \pi_{\theta, \lambda}(\uhatbold | c) 
= \log p_\theta(\zhatbold) - \sum_{i=1}^K \log \Det{\pderiv{\yhatbold_i}{\yhatbold_{i-1}}}.
\label{eq:nf_log_likelihood}
\end{equation}
In this work, we make use of the affine coupling layer proposed by \citet{dinh2016density} as part of the real non-volume-preserving (RealNVP) flow.
The core idea is to split the input $\uhatbold$ into two partitions $\uhatbold = (\uhatbold_{I_1}, \uhatbold_{I_2})$, where $I_1$ and $I_2$ are a partition of $[1, MH]$, and apply
\begin{equation}
\begin{aligned}
\yhatbold_{I_1} = \uhatbold_{I_1}, \quad
\yhatbold_{I_2} = \uhatbold_{I_2} \odot \exp{s_{\lambda}(\uhatbold_{I_1}, c)} + t_{\lambda}(\uhatbold_{I_1}, c),
\end{aligned}
\end{equation}
where $s_{\lambda}$ and $t_{\lambda}$ are the scale and translation terms, which are represented with arbitrary neural networks, and $\odot$ is the Hadamard product. 
This makes computing the log-determinant term in \Cref{eq:nf_log_likelihood} and inverting the flow fast and efficient.
Now, in robotics, we often have lower and upper limits on the controls.
These are usually enforced in sampling-based MPC by either clamping the control samples or passing them through a scaled sigmoid.
However, instead of enforcing the constraints heuristically after sampling, we learn a constrained sampling distribution directly.
Since the sigmoid function is invertible and has a tractable log-determinant (shown in \Cref{app:sigmoid}), we can simply append one after $h_{\lambda_K}$ in the NF and scale it by the control limits.
This ensures that control constraints are satisfied by design and taken into account while learning the distribution.

\vspace{-1ex}
\subsection{Training the Sampling Distribution}
\label{sec:training}
\vspace{-0.5ex}
Computing gradients through the upper-level objective is not straightforward, as both the expectation and the inner terms of \Cref{eq:bilevel-objective} depend on $\lambda$.
Therefore, the state distribution depends on the NF and latent shift model.
One way around this issue is to consider a modified objective at each batch $d$:
\begin{equation}
\ell_d(\thetabold, \lambda; c) = \Expectt{\pibold_{\thetabold,\lambda_d}, \rho, f}{
\sum_{t=0}^{T-1} \Jhat(\theta_t, \lambda; x_t, c)
},
\end{equation}
which fixes the outer expectation to be with respect to the current policy.
Intuitively, this choice trains the NF to optimize the MPC cost function under the state distribution resulting from the current policy $\pibold_{\thetabold,\lambda_d}$.
This effectively ignores the dependency of the action chosen to interact with the true environment and the learned parameters.
We then update the outer expectation distribution at each episode, overcoming the covariate shift problem that would otherwise arise.

Now, we only have to focus on computing the gradient $\nabla_\lambda \Jhat(\theta_t(\lambda), \lambda; x_t, c)\rvert_{\lambda=\lambda_d}$ for each time step, which can be computed similar to \Cref{eq:grad_Jhat} and approximated with Monte Carlo sampling:
\begin{equation}
\nabla \Jhat(\theta_t(\lambda), \lambda; x_t, c) \approx - \sum_{i=1}^N w_i \nabla_\lambda \log \pi_{\theta_t(\lambda), \lambda}(\uhatbold_t^{(i)} | c),
\end{equation}
where the weights $w_i$ are defined according to \Cref{eq:MPPI_mean_update_approx}.
The log-likelihood is given by \Cref{eq:nf_log_likelihood}, the gradient of which involves computing the backwards pass through the network $h_\lambda$. 
However, we also have to consider the dependence of the latent distribution parameters $\thetabold(\lambda)$ on $\lambda$.
Therefore, we must backpropagate through the MPC update: 
\begin{equation}
\mubold_{t}(\lambda) = (1-\gamma_t) \mutildebold_{t}(\lambda) + \gamma_t \Delta \mubold_{t},\quad
\Delta \mubold_{t} = 
\frac{\Expect{\pi_{\thetatilde_{t}(\lambda), \lambda}, \fhat}{e^{ -\frac{1}{\beta} C(\xhatbold_t, \uhatbold_t) } h_{\lambda}(\uhatbold_{t}; c) }}{\Expect{\pi_{\thetatilde_{t}(\lambda), \lambda}, \fhat}{e^{ -\frac{1}{\beta} C(\xhatbold_t, \uhatbold_t)} }}
\label{eq:latent_mppi_update}
\end{equation}
where the previous shifted mean $\mutildebold_{t}(\lambda)$ is given by the learned latent shift model.
Note that we have rewritten the expectations in terms of the control distribution, rather than the latent distribution. 
This is necessary in order to derive the following approximate gradient without requiring differentiability.
To compute the gradient of \Cref{eq:latent_mppi_update}, we must approximate the gradient of $\Delta \mubold_{t}$ with respect to $\lambda$, which we can compute as $\pderiv{\Delta \mubold_t}{\lambda} \approx M_1 - M_2 M_3$, where we define:
\begin{equation}
\begin{gathered}
M_1 = \sum_{i=1}^N w_i \Big[ \nabla_\lambda h_\lambda(\uhatbold_t^{(i)}; c) + h_\lambda(\uhatbold_t^{(i)}; c) \nabla_\lambda \log \pi_{\thetatilde(\lambda),\lambda}(\uhatbold_t^{(i)} | c) \Big],\\
M_2 = \sum_{i=1}^N w_i h_\lambda(\uhatbold_t^{(i)}; c),\quad 
M_3 = \sum_{i=1}^N w_i \nabla_\lambda \log \pi_{\thetatilde(\lambda),\lambda}(\uhatbold_t^{(i)} | c).
\end{gathered}
\end{equation}
The derivation of this approximate gradient can be found in \Cref{app:approx_gradient_proof}.
Note that computing the gradients $\nabla_\lambda \log \pi_{\thetatilde(\lambda),\lambda}(\uhatbold_t^{(i)} | c)$ will also require us to backpropagate through the shift model due to the dependence of $\thetatilde$ on $\lambda$.
Therefore, even when the step size is set to one, i.e. $\gamma_t = 1$, we introduce a form of recurrence between time steps.
Please see \Cref{app:alg} for additional visualizations of the computational graph generated by an episode and further descriptions of the overall algorithm.

\vspace{-1.5ex}
\section{Related Work}
\vspace{-1ex}
There have been many recent attempts at combining machine learning with MPC, which mostly center around learning or fine-tuning a good dynamics model \citep{erickson2018deep, williams2017information, lenz2015deepmpc, kocijan2004gaussian, fu2016one, chua2018deep, nagabandi2018neural}, potentially from high-dimensional observations \citep{finn2017deep, wahlstrom2015pixels, watter2015embed, banijamali2018robust, ebert2018visual, ha2019adaptive, hafner2019learning}.
Alternatively, some methods target the cost function, learning terminal value functions \citep{zhong2013value, rosolia2017learning, lowrey2018plan, bhardwaj2020blending, bhardwaj2020information} or cost-shaping terms \citep{tamar2017learning} to improve performance with reduced prediction horizons.
Instead of learning dynamics and cost separately, other methods propose learning the entire optimal controller \citep{amos2020differentiable, karkus2017qmdp, okada2017path, amos2018differentiable, okada2018acceleration, pereira2018mpc, tamar2017learning} or planner \citep{tamar2016value, srinivas2018universal, yu2019unsupervised, bhardwaj2020differentiable} end-to-end.
In contrast, recent work has explored leaving the dynamics and cost function fixed and instead focused on improving the optimization process.
For example, \citet{sacks2022learning} propose to use learning to replace the update rule employed to solve the optimization problem at each time step with a neural network.
They showed that this can substantially reduce the number of samples required to achieve a target task performance with sampling-based MPC.

Using alternative sampling distributions is another viable strategy for improving the optimization process in MPC.
Multiple works have considered sampling distributions beyond simple Gaussians, such as Gaussian mixture models \citep{okada2020variational} and a particle method based on Stein variational gradient descent (SVGD) \citep{lambert2020stein}.
Additionally, most implementations use heuristics to modify samples and squeeze out additional performance gains \citep{bhardwaj2021fast, pinneri2020sample}.
In terms of learning the distribution, \citet{amos2020differentiable} learn a latent action space for their proposed differentiable cross-entropy method (CEM) controller.
However, they require all components of the pipeline to be differentiable and do not consider learning a shift model.
\citet{agarwal2020imitative} learn a normalizing flow in the latent space of a variational autoencoder (VAE). 
Yet, they also require differentiability, use expert demonstrations to learn the latent space of the VAE, and have no means of warm-starting between time steps.
\citet{wang2019exploring} propose to use a learned feedback policy to warm-start MPC, but still rely on a Gaussian perturbations to the proposed action sequence.
The authors also explore performing online planning in the network's parameter space, which results in a massive action space that may be hard to scale.

\citet{power2021variational, power2022variational} also train a normalizing flow to use as the sampling distribution for MPC, but they do not learn a latent shift model and perform all operations in the control space.
Moreover, they mix the latent samples with Gaussian perturbations to the current control-space mean, as they do not update the latent distribution directly.
This prevents them from fully taking advantage of the learned distribution and throws away useful information which could potentially improve performance.
In fact, we show in \Cref{sec:eval_shift} that the learned shift model contributes significantly to the performance gains.
Additionally, a primary focus of their work is on how to handle out-of-distribution (OOD) environments by learning a posterior over environment context variables.
We could combine their approach with ours by conditioning the learned shift model on the inferred environment context for improved generalization.
Finally, normalizing flows have also been used to improve exploration in RL \citep{tang2018boosting, tang2018implicit, mazoure2020leveraging} by providing a more flexible, and potentially multimodal, distribution. 
They have also been employed in sampling-based motion planning \citep{lai2020learning, lai2021parallelised} to provide good proposal configurations to speed up convergence.

\vspace{-2ex}
\section{Experimental Results}
\vspace{-1.5ex}
In all experiments, we denote our proposed approach as \nfmpc, the baseline MPPI implementation as \mppi, and the method by \citet{power2021variational, power2022variational} as \flowmppi.
Details about the hyperparameters, implementation, tasks, and training can be found in \Cref{app:exp_details}.
We evaluate on a fixed set of environments, which includes start states, goal locations, and obstacle placements, and run 32 rollouts for each sample amount.
Our primary metrics for comparison are the success rate, defined as the percentage of times the task goal was achieved, and the average cost of trajectories which successfully completed the task.
Additionally, we cannot ignore the overhead introduced by the NF in the control pipeline.
Therefore, we measured the average change in wall clock time across different amounts of samples for \nfmpc\ and \flowmppi\ in \Cref{app:timing}. 

\vspace{-1.5ex}
\subsection{Planar Robot Navigation}
\vspace{-1ex}

We begin by applying \nfmpc\ to a planar robot navigation task in which a 2D holonomic point-robot must reach a goal position while avoiding eight dynamic obstacles, which move around randomly at each time step (\pnranddyn).
The point-robot has double integrator dynamics with stochasticity on the acceleration commands to create a mismatch between the predictive model used by MPC and the true environment.
Each obstacle's current position is perturbed with Gaussian noise and clipped to be within map bounds, and the starting and goal locations of the robot are randomized in each episode.
The NF for both \nfmpc\ and \flowmppi\ is conditioned on the obstacle locations, current state, and goal position.
Note that we do not condition the shift model, as this consistently hurt performance.

\begin{figure}
    \centering
    \includegraphics[width=\textwidth]{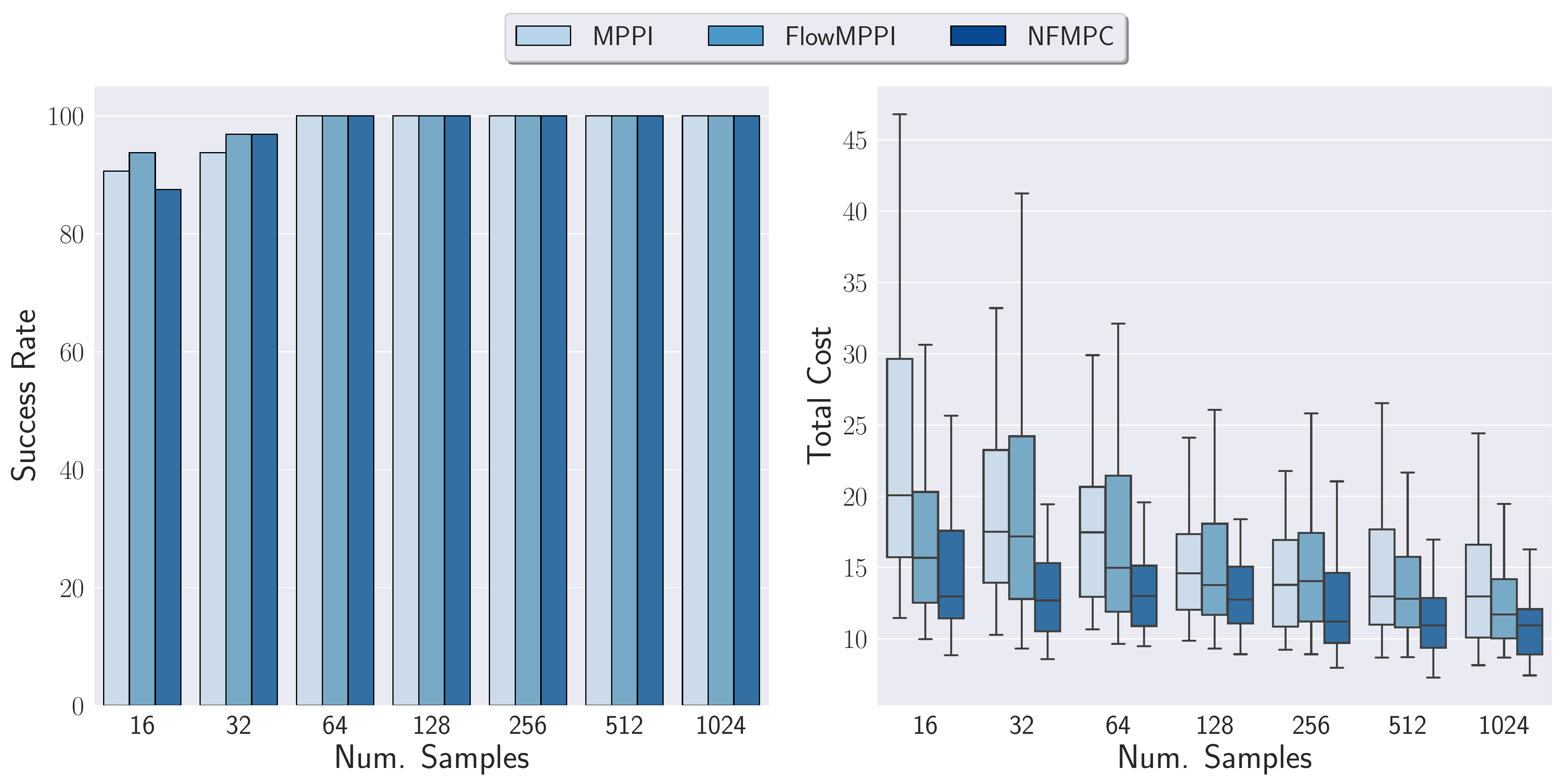}
    \vspace{-3ex}
    \caption{Success rate and cost distribution on the \pnranddyn\ environment across a different number of samples.}
    \label{fig:spnranddyn_barchart}
    \vspace{-2ex}
\end{figure}

\begin{figure*}[t]
    \begin{minipage}[t]{\textwidth}
        \raisebox{3.8ex}{\rotatebox[origin=t]{90}{NFMPC}}
        \includegraphics[width=0.98\textwidth]{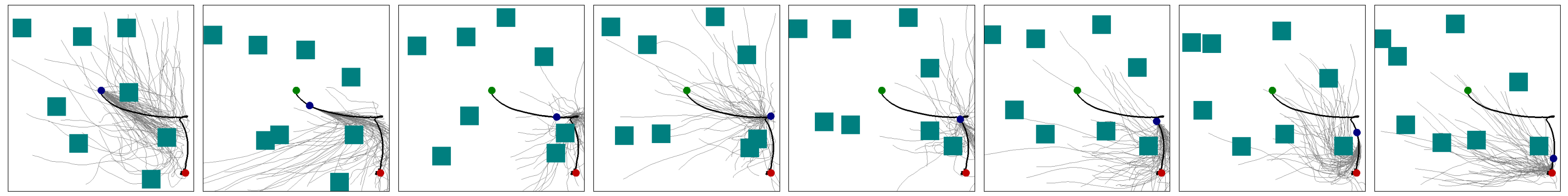}
    \end{minipage}
    \begin{minipage}[t]{\textwidth}
        \raisebox{3.8ex}{\rotatebox[origin=t]{90}{FlowMPPI}}
        \includegraphics[width=0.98\textwidth]{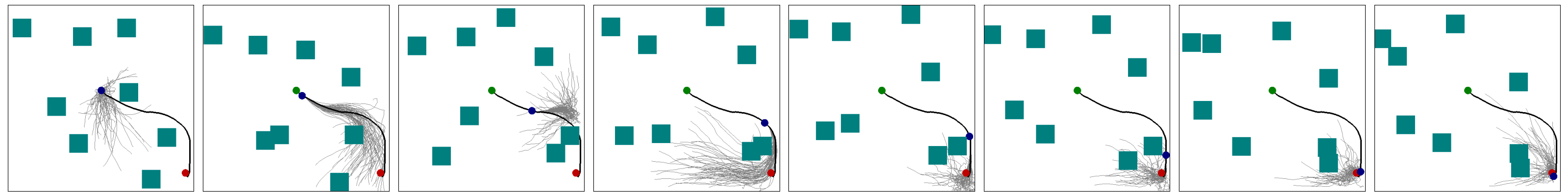}
    \end{minipage}
    \begin{minipage}[t]{\textwidth}
        \raisebox{3.8ex}{\rotatebox[origin=t]{90}{MPPI}}
        \includegraphics[width=0.98\textwidth]{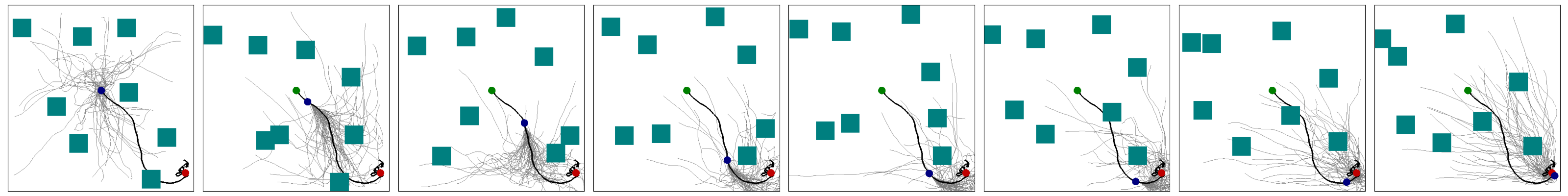}
    \end{minipage}
    \caption{Visualization of a trajectory and top samples from (top) \nfmpc, (middle) \flowmppi, and (bottom) \mppi\ on the \pnranddyn\ task.}
    \label{fig:spnranddyn_trajectories}
    \vspace{-3ex}
\end{figure*}

We quantitatively compare all controllers in \Cref{fig:spnranddyn_barchart}. 
The trajectory cost box plots represent the median and quartiles of the distribution.
We find that \nfmpc\ consistently outperforms the both \mppi\ and \flowmppi.
While all controllers reach a 100\% success rate at 1024 samples, \nfmpc\ achieves a 20\% and 8\% lower median cost over \mppi\ and \flowmppi, respectively.
We also find that \nfmpc\ scales more gracefully overall as the number of samples is reduced.
For instance, it is able to withstand a $32\times$ decrease in the number of samples (1024 to 32) while reducing success rate by only 3\% and increasing median cost by 8\%, nearly matching \flowmppi\ at 1024 samples.
In comparison, \mppi\ reduces success rate by 6\% and increases median cost by 43\%.
Similarly, \nfmpc\ outperforms \flowmppi\ at all sample amounts in terms of median cost, although at 16 samples \flowmppi\ achieves a slightly higher success rate than \nfmpc.
We visualize trajectories and top samples in \Cref{fig:spnranddyn_trajectories}.
The green and red dots are the starting and goal locations, respectively, and the blue dot is the current position of the robot at the given time step.
The thick black line is the resulting path taken by the controller, while the gray lines are the top samples generated at the current state.
In this example, \flowmppi\ nearly collides with an obstacle, while \nfmpc\ and \mppi\ are able to safely reach the goal.
\flowmppi\ over commits to a narrow corridor and is unable to reroute in time to account for the new obstacle location.
\nfmpc\ takes a similar trajectory to \flowmppi, however, it is able to pause until the obstacle moves out of the way to proceed towards the goal.
We also provide additional experiments with static environments and comparing unconditional and conditional models in \Cref{app:pn_exp}.

\begin{figure}
    \centering
    \includegraphics[width=\textwidth]{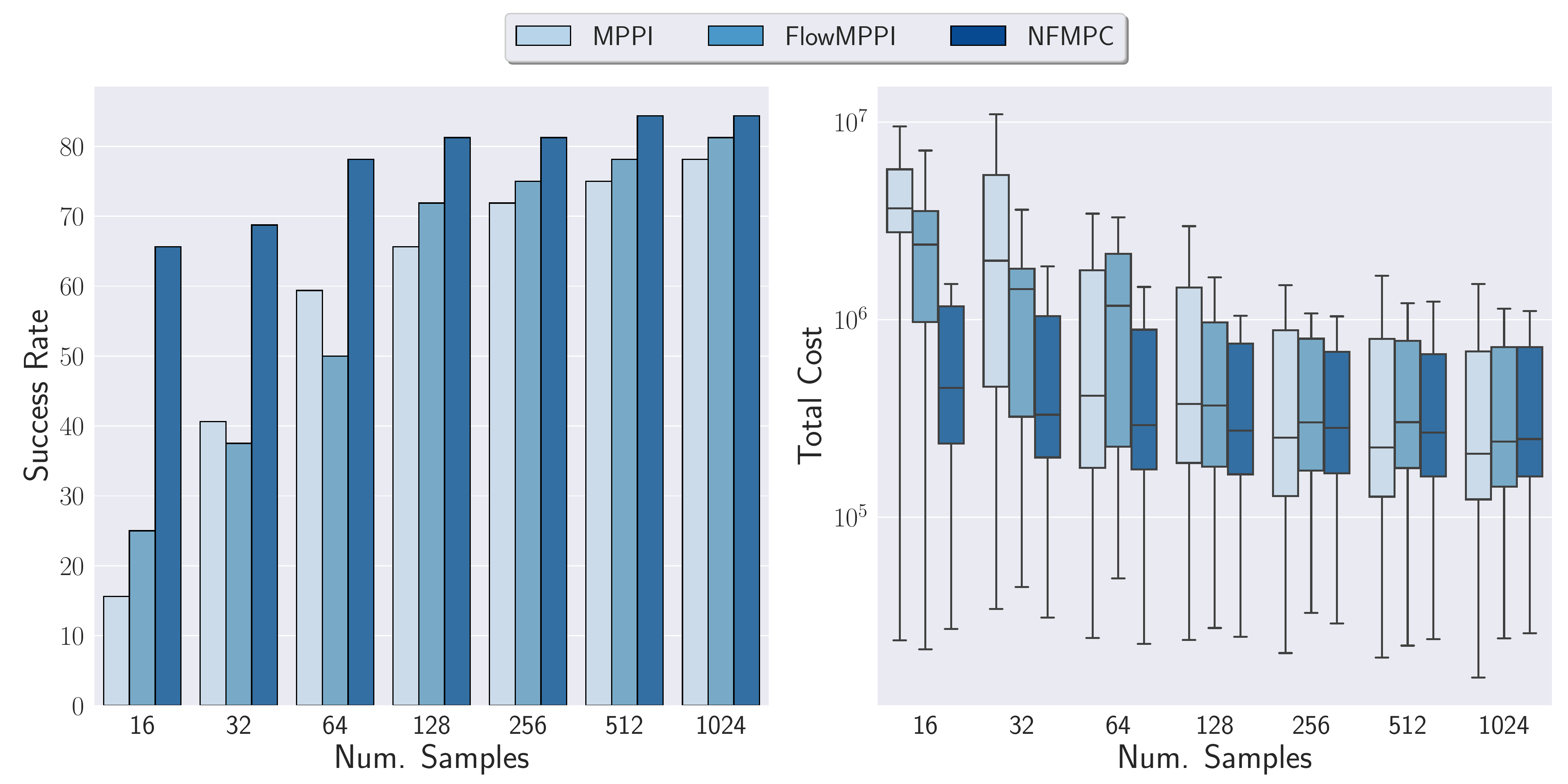}
    \vspace{-3ex}
    \caption{Success rate and cost distribution on the \frankaobs\ environment across a different number of samples.}
    \vspace{-2ex}
    \label{fig:frankaobs_barchart}
\end{figure}

\begin{figure*}[t]
    \begin{minipage}[t]{\textwidth}
        \raisebox{3.8ex}{\rotatebox[origin=t]{90}{NFMPC}}
        \includegraphics[width=0.98\textwidth]{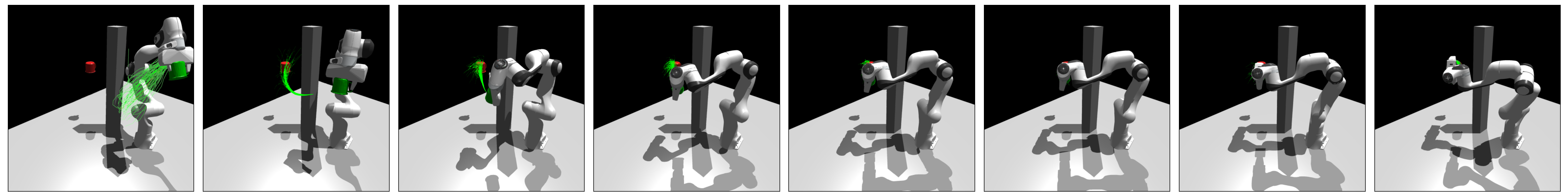}
    \end{minipage}
        \begin{minipage}[t]{\textwidth}
        \raisebox{3.8ex}{\rotatebox[origin=t]{90}{FlowMPPI}}
        \includegraphics[width=0.98\textwidth]{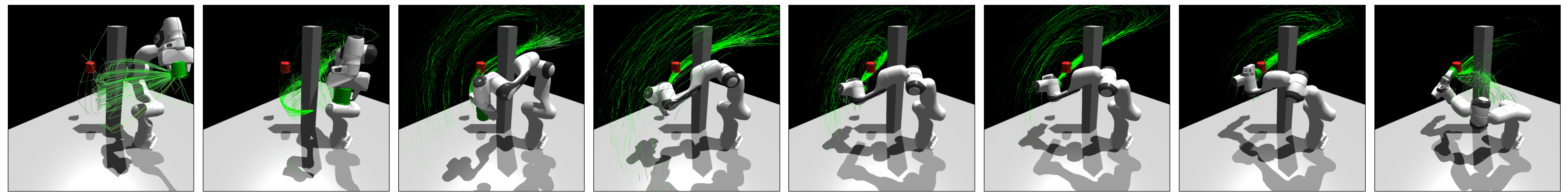}
    \end{minipage}
        \begin{minipage}[t]{\textwidth}
        \raisebox{3.8ex}{\rotatebox[origin=t]{90}{MPPI}}
        \includegraphics[width=0.98\textwidth]{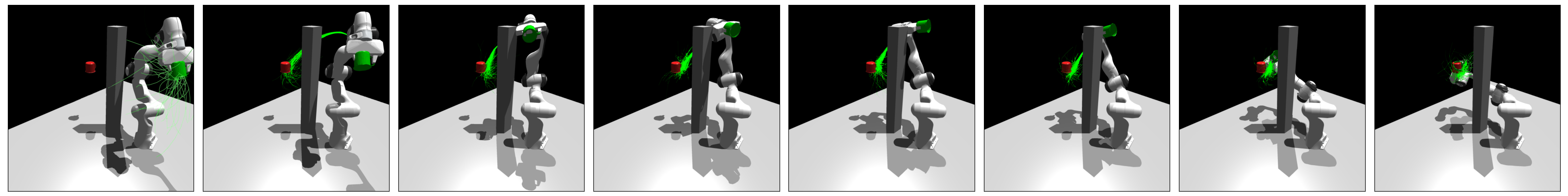}
    \end{minipage}
    \caption{Visualization of a trajectory and top samples from (top) \nfmpc, (middle) \flowmppi, and (bottom) \mppi\ on the \frankaobs\ task.}
    \label{fig:franka_obstacles_trajectories}
    \vspace{-3ex}
\end{figure*}

\vspace{-2.5ex}
\subsection{Franka Panda Arm}
\vspace{-1ex}

Next, we apply \nfmpc\ to the \franka\ task, which involves controlling a 7 degree-of-freedom (DOF) Franka Panda robot arm and steering it towards a randomly placed target goal from a fixed starting pose while avoiding a single pole obstacle.
The NF for both \nfmpc\ and \flowmppi\ is conditioned on the obstacle locations, initial state, and goal position.
It is also important to note that no controller achieves a $100\%$ success rate, as not every randomly generated environment is feasible.
As shown in \Cref{fig:frankaobs_barchart}, at 1024 samples, \nfmpc\ achieves a success rate of $84\%$, while \flowmppi\ and \mppi\ only succeed $81\%$ and $78\%$ of the time, respectively.
Again, \nfmpc\ scales better with a reduced number of samples.
With a $64\times$ decrease in the number of samples (1024 to 16), \nfmpc\ only drops in success rate by $29\%$.
Meanwhile, \flowmppi\ decreases by nearly $70\%$ to a success rate of $25\%$.
However, both fair better than \mppi, which drops by $80\%$ to a success rate of $16\%$.
These results support the hypothesis that learning to perform MPC updates in the latent space of the NF and training the controller as a recurrent network improves performance.

To more clearly understand what \nfmpc\ is doing differently, we visualize the performance of all three controllers on a held-out validation environment in \Cref{fig:franka_obstacles_trajectories}.
The green and red markers indicate the end-effector and goal positions, respectively, while the single pole is the obstacle which must be avoided.
Additionally, the green trajectories represent the top samples from the controller at each time step.
While \mppi\ collides with the obstacle, both \nfmpc\ and \flowmppi\ learn to take a different path which is collision-free.
\flowmppi\ generates better initial trajectories than \nfmpc, which are more pointed towards the goal location and achieve a greater velocity.
However, \nfmpc\ is able to better adapt the sampling distribution throughout the episode and reach the goal more quickly.
Finally, in \Cref{app:franka_exp}, we present results for the environment with no obstacles, a comparison of conditional and unconditional models, and a breakdown of the individual cost terms for all models to gain insight into the learned sampling distributions.
We also perform an ablation which removes the learned shift model from \nfmpc\ to illustrate that it is a crucial component.

\vspace{-2ex}
\section{Limitations}
\vspace{-1.5ex}
A major limitation of \nfmpc\ is that the learned distribution and shift model are only valid for a fixed horizon and control dimensionality.
Therefore, these components cannot be directly transferred to new robots or for alternate horizons without being retrained.
However, this could potentially be remedied by novel architectural innovations and training distributions across both environments and robots.
Moreover, the learned distribution is specific to the environmental distributions on which it was trained.
Therefore, it does not always perform as well when transferred to out-of-distribution environments.
However, this is always going to be a challenge for any learning-based method and addressing it is an open problem for future research.
For instance, our approach could be combined with the method developed by \citet{power2021variational, power2022variational}, which learned a generative model of environments and used this learned distribution to perform a projection step on the conditional NF.
However, this approach would only be valid for conditional models, and addressing transfer of unconditional models is an unresolved question.
Finally, both our approach and \flowmppi\ introduce an additional overhead to due to running the NF which cannot be ignored.
This increase in wall clock time may be worth the additional performance gains.
Additionally, we scale better with a reduction in samples compared to all baselines.
Therefore, we can potentially alleviate some of the introduced overhead by reducing the number of samples while still meeting the application demands.

\vspace{-2ex}
\section{Conclusion}
\vspace{-1.5ex}
We presented a method for learning MPC sampling distributions with normalizing flows (NFs) which moves all online parameter updates and warm-starting operations into the latent space.
We show how to frame the problem as bi-level optimization and derive an approximate gradient through the MPC update to train the distributions.
Additionally, we illustrate how to incorporate control box constraints directly into the NF architecture.
Through our empirical evaluations in both simulated navigation and manipulation problems, we demonstrate that our approach is able to surpass the performance of all baselines.
Moreover, we find that controllers which move all operations into the latent space are often able to scale more gracefully with a reduction in the number of samples.
These results indicate the importance of leveraging the latent space in learned sampling distributions for MPC.
Finally, because we learn all components of the controller through episodic interactions with the environment, they can potentially be trained to account for the modeling errors in the MPC controller.
\vspace{-2ex}



\bibliography{references}

\newpage
\onecolumn
\appendix
\section{Appendix}

\vspace{-1.5ex}
\subsection{Computational Graph Visualization}
\vspace{-1ex}
\label{app:alg}

\begin{wrapfigure}{r}{0.5\textwidth}
\centering
\vspace{-3ex}
\includegraphics[width=0.48\textwidth]{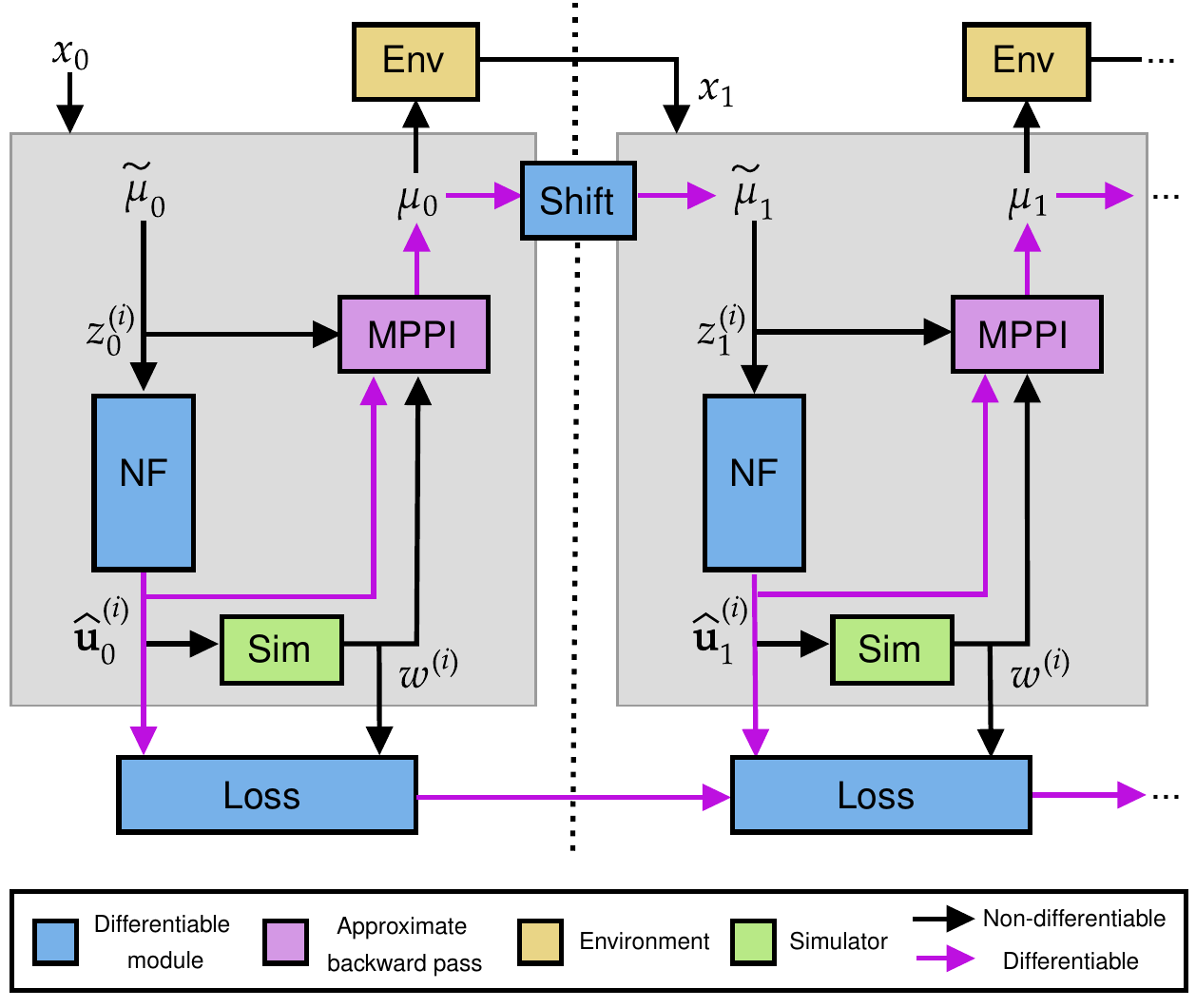}
\caption{Computational graph of an episode which illustrates the interaction between the \textcolor{myblue}{normalizing flow (NF)}, learned \textcolor{myblue}{latent shift model (Shift)}, the \textcolor{mypurple}{MPPI update (MPPI)} of the latent mean, the \textcolor{mygreen}{simulator (Sim)}, and the \textcolor{myyellow}{environment (Env)}.}
\label{fig:arch}
\vspace{-2ex}
\end{wrapfigure}

We visualize the computational graph of a single episode in \Cref{fig:arch}.
The \textcolor{myblue}{normalizing flow (NF)}, \textcolor{myblue}{shift operator (Shift)}, and \textcolor{myblue}{loss computation (Loss)} are colored blue, indicating that they are differentiable modules.
The \textcolor{mypurple}{MPPI update (MPPI)} is colored in purple, to indicate that we are approximating the gradient as above, which can be implemented as a custom backwards pass in autodifferentiation software.
All of the paths of the graph through which the gradients flow during the backwards pass are colored in purple.
The \textcolor{mygreen}{simulator (Sim)}, colored in green, is solely used to compute the weights for the MPPI update, which are reused in the backwards pass.
At each time step, the updated mean $\mubold_t$ defines a latent Gaussian, which is used to generate controls applied to the actual \textcolor{myyellow}{environment (Env)}, denoted by the yellow box.
Note that both the latent variables and controls are being passed to the \textcolor{mypurple}{MPPI} module. 
This is because during the forward pass, we can simply take the weighted sum of latent variables.
However, during the backward pass, we have to re-run the network with the controls to compute the approximate gradients, as discussed in \Cref{sec:training}.

\vspace{-2ex}
\subsection{Experimental Details}
\vspace{-0.5ex}
\label{app:exp_details}

\paragraph{Controller Details.}
\begin{table}[]
    \centering
    \caption{Controller Hyperparameters}
    \begin{tabular}{@{\extracolsep{5pt}}l*{2}{c@{\enspace}c@{\enspace}c}}
        \hrulethick
        {} & {} & \multicolumn{3}{c}{\textbf{Environment}} \\
        \cmidrule{3-5}
        \multicolumn{2}{c}{\textbf{Parameter}} & \pngrid\ & \pnrand\ & \franka\ \\
        \midrule
        \multicolumn{2}{c}{Horizon (H)} & 32 & 64 & 32\\
        \multicolumn{2}{c}{Temperature ($\beta$)} & $10^{-32}$ & $10^{-32}$ & $10^{-4}$\\
        \midrule
        \multicolumn{1}{c|}{\multirow{4}{*}{MPPI}} & Init. cov. ($\sigma^2$) & 10  & 100 & 0.1\\
        \multicolumn{1}{c|}{} & Step size ($\gamma$) & 0.7 & 1 & 1\\
        \multicolumn{1}{c|}{} & Spline knots ($n$) & None & None & 4 \\
        \midrule
        \multicolumn{1}{c|}{\multirow{3}{*}{FlowMPPI}} & Init. cov. ($\sigma^2$) & 10  & 10 & 0.1\\
        \multicolumn{1}{c|}{} & Latent cov. & 1  & 1 & 0.1\\
        \multicolumn{1}{c|}{} & Latent mean penalty ($\lambda$) & $10^{-4}$ & $10^{-3}$ & 1\\
        \midrule
        \multicolumn{1}{c|}{\multirow{1}{*}{NFMPC}} & Latent cov. & 1  & 1 & 0.1\\
        \hrulethick
    \end{tabular}
    \label{tbl:mppi_hp}
    \vspace{-2ex}
\end{table}

We use a modified version of the MPPI implementation by \citet{bhardwaj2021fast}, which is implemented in PyTorch \citep{paszke2019pytorch}.
For the MPPI baseline, we perform covariance adaptation of the full covariance matrix across the horizon and control dimensions.
The initial covariance matrix is always an identity matrix scaled by an initial covariance scalar hyperparameter $\sigma^2$.
Additionally, this implementation uses Halton sequences \citep{halton1964algorithm} for generating control sequence samples and smooths the sampled trajectories with $n$-degree B-splines in some tasks.
When B-splines are used, we sample the Halton sequence once at the beginning of a rollout and then transform it using the mean and covariance of the Gaussian distribution.
Additionally, we ensure that the mean of the Gaussian is always in the set of samples.
All hyperparameters of the controllers were chosen via a grid search and the final choices for each task are listed in \Cref{tbl:mppi_hp}.
When it improves performance, we warm-start the controllers prior to the first time step by running the MPC update for 100 iterations to ensure convergence.
Additionally, we normalize the total trajectory costs prior to computing the softmax weights, as discused by \citet{okada2020variational}.

In general, we use the same settings for \nfmpc\ and \flowmppi\ that were found in this grid search.
However, we do not performance covariance adaptation on the latent Gaussian and assume the flow learns how to adjust sample spread as needed.
Additionally, we take the previous mean in the control space, shift it forward, and add it to the set of control samples at the next time steps.
While the learned latent shift model handles this well in most cases, we found adding this sample sped up training and slightly improved performance.
We always use Halton sequences to sample from the latent Gaussian, but never use B-splines.
For \flowmppi, we always use half the samples for sampling from the NF and half for the Gaussian perturbations on the current control-space mean.
We do perform covariance adaptation on the perturbation Gaussian and re-tune its initial covariance.
Finally, we have the additional $\lambda$ parameter, which penalizes the latent Gaussian samples from deviating too much from the projection of the current control-space mean into the latent space.

\vspace{-1ex}
\paragraph{Planar Robot Navigation.}
The planar navigation environment has a state space of $x_t \in \real^4$, which consists of the robot's 2D position, $(p_x, p_y)$, and velocity, $(v_x, v_y)$, and a control space of $u_t \in \real^2$, which are the robot's 2D acceleration commands.
The robot has double-integrator dynamics with additive Gaussian noise on the controls, as described by the following equations:
\begin{equation}
\begin{bmatrix}
p_x \\ p_y \\ v_x \\ v_y
\end{bmatrix}_{t+1}
=
\begin{bmatrix}
1 & 0 & \Delta t & 0 \\
0 & 1 & 0 & \Delta t \\
0 & 0 & 1 & 0 \\
0 & 0 & 0 & 1
\end{bmatrix}
\begin{bmatrix}
p_x \\ p_y \\ v_x \\ v_y
\end{bmatrix}_t
+
\begin{bmatrix}
0 & 0 \\
0 & 0 \\
\Delta t & 0 \\
0 & \Delta t 
\end{bmatrix}
(u_t + w_t),\quad
w_t \sim \mathcal{N}(0, \sigma \eye),
\end{equation}
where we set $\Delta t=0.1$ and $\sigma=1$. 
Additionally, we added acceleration limits of $\barbelow{u}=-10$ and $\bar{u}=10$ for both directions.
The cost function consists of the Euclidean distance to the goal, a signed-distance field representation of the obstacles, and a term which encourages the robot to stay within the bounds of the map, and a quadratic control penalty:
\begin{equation}
c(x, u) = w_{goal} ||x - x_g||_2^2 + w_{bound} c_{bound}(x_t) + w_{coll} \text{SDF}(p_x, p_y) + w_{ctrl} ||u||_2^2,
\end{equation}
where we define the map bound cost as:
\begin{equation}
c_{bound}(x) = \sum_{i \in (x, y)}
\ind{(p_i > \bar{p}_i)\ ||\ (p_i < \barbelow{p}_i)}
\min{\big( (p_i - \bar{p}_i)^2, (p_i - \barbelow{p}_i)^2 \big)}.
\end{equation}
In the above equations, $x_g$ is the goal state, $\bar{p}_i$ and $\barbelow{p}_i$ are the upper and lower bounds of the map for each coordinate, and $\text{SDF}(\cdot, \cdot)$ indexes an image which represents the signed-distance field.
We mainly focus on the \pnranddyn\ task, which involves steering the robot towards a goal position while avoiding eight dynamic obstacles.
While the obstacles drift randomly in the environment with Gaussian steps, their positions are clipped to be within map bounds.
We do not update their position if the perturbation would bring it too close to the robot or goal location to prevent collisions which the robot cannot react to in time to avoid.
An episode lasts for 200 time steps and is considered successful if the agent reaches the goal without colliding into any obstacles.
In \Cref{app:pn_exp}, we also consider a static version of this environment (\pnrand), which simply places the eight obstacles randomly in the environment, and a version which arranges the obstacles in a fixed grid (\pngrid).

\vspace{-1ex}
\paragraph{Franka Panda Arm.}
The Franka Panda arm environment defines the robot state in joint space with $x_t \in \real^{21}$, consisting of each joint's angle $\theta_i$, angular velocity $\dot{\theta}_i$, and angular acceleration $\Ddot{\theta}_i$.
Its control space is $u_t \in \real^7$, which are the angular acceleration commands for each joint.
The dynamics are deterministic and implemented by the Nvidia Isaac Gym simulator \citep{makoviychuk2021isaac}.
However, the MPC controllers use a simpler kinematic model defined by \citet{bhardwaj2021fast}, which is implemented in a batch fashion by leveraging its linearity:
\begin{equation}
\bm{\Ddot{\Theta}} = \bm{u},\quad
\bm{\dot{\Theta}} = 
\dot{\Theta}_t + S_l(1) \diag{\Delta t} \bm{\Ddot{\Theta}},\quad
\bm{\Theta} = 
\Theta_t + S_l(1) \diag{\Delta t} \bm{\dot{\Theta}},
\end{equation}
where the bold symbols indicate that they consist of values along the entire horizon, $\bm{\Theta}$ includes the angles for all joints, $S_l(1)$ is a lower triangular matrix filled with 1, and $\Delta t$ is a vector of time steps across the horizon. 
We use smaller time steps earlier along the horizon and larger ones for later time steps.
By implementing the dynamics in batch, we avoid iteratively unrolling the dynamics, speeding up controller computation significantly.
For computing cost, we also require the Cartesian poses $X$, velocities $\dot{X}$, and accelerations $\Ddot{X}$ of the end-effector.
These are obtained via:
\begin{equation}
X = \text{FK}(\Theta),\quad
\dot{X} = J(\Theta)\dot{\Theta},\quad
\Ddot{X} = \dot{J}(\Theta) \dot{\Theta} + J(\Theta) \Ddot{\Theta}
\end{equation}
where $\text{FK}(\Theta)$ are the forward kinematics and $J(\Theta)$ is the kinematic Jacobian.
The cost function is a weighted sum of a number of terms, as defined by \citet{bhardwaj2021fast}, which includes: distance of end-effector to the goal pose $c_{pose}$, a time varying velocity limit cost $c_{stop}$ that enables stopping within the specified horizon, a joint limit cost $c_{joint}$, a manipulability cost $c_{manip}$ which encourages the arm to avoid singular configurations, a self-collision cost $c_{self-coll}$, and a obstacle collision cost $c_{coll}$.
The overall final cost function is then:
\begin{equation}
\small
c(x, u) = w_{p} c_{pose}(x) + w_{s} c_{stop}(x) + w_{j} c_{joint}(x) + w_{m} c_{manip}(x) + w_{c} (c_{self-coll}(x) + c_{coll}(x)).
\end{equation}
The self-collision cost is implemented with a neural network that predicts the closest distance between the links of the robot given a configuration.
The collision cost is a binary cost which uses a learned collision checking function that operates directly on raw point cloud data and classifies if a robot link is in collision.
See \citet{bhardwaj2021fast} for further details about each of the individual cost terms.
For the \frankaobs\ task, we control the 7 degree-of-freedom (DOF) Franka Panda robot arm and steer it towards a target goal from a fixed starting pose while avoiding a single pole obstacle.
The obstacle and goal positions are randomized at the beginning of each episode, which lasts for 600 time steps.
An episode is considered successful if the end effector reaches the target position under the time constraints while avoiding the obstacle.
In \Cref{app:franka_exp}, we also consider a simplified version which has no obstacles (\franka).

\vspace{-1ex}
\paragraph{Architectural and Training Details.}
All NFs for both \nfmpc\ and \flowmppi\ were implemented in PyTorch and contain affine coupling layers which use multilayer perceptrons (MLPs) for both the scale and translation terms.
The scale and translation networks use Tanh and ReLU activations, respectively.
We also employ layer normalization \citep{ba2016layer} in these networks to help prevent overfitting.
Interestingly, we found that adding batch normalization between each layer, as proposed by \citet{dinh2016density}, actually hurt performance and was therefore excluded.
In all environments, we use 5 RealNVP blocks for the NF, and each MLP has a hidden dimensionality of 128 neurons.
For the \pngrid, \franka, and \frankaobs\ tasks, the shift model is also an MLP with a single hidden layer of 128 neurons and a ReLU activation function.
In \pnrand\ and \pnranddyn\, the shift model is implemented as an LSTM with a hidden dimensionality of 128 neurons.
For a fair comparison, the same architecture was used for both \flowmppi\ and \nfmpc.
We trained all networks with the Adam optimizer \citep{kingma2014adam} using a learning rate of $10^{-4}$.

\begin{wrapfigure}{r}{0.54\textwidth}
\vspace{-2.5ex}
\begin{minipage}{0.54\textwidth}
\begin{algorithm}[H]
\SetAlgoLined
\small
\KwIn{Environment dist. $\Ccal$, initial state dist. $\rho$, initial param. $\thetatilde_0$, $\lambda_0$}
\KwParam{\# episodes $D$, episode length $T$, \# samples $N$}

\For{$d = 1, 2, \dots, D$}{
    Sample environment $c \sim \Ccal(\cdot)$ \\
    Sample initial state $x_0 \sim \rho(\cdot | c)$ \\
    Initialize episode loss $l_d$ $\leftarrow$ $0$ \\
    \For{$t = 0, 1, \dots, T_k-1$}{
        $(\uhatbold_t, \zhatbold_t, w_t)^{(1:N)}$ $\leftarrow$ Rollout($x_t$, $c$, $\thetatilde_t$, $\lambda_d$) \\
        Update $\thetatilde_t$ to $\theta_t$ using latent MPPI update \\
        Sample $u_t \sim \pi_{\theta_t, \lambda_d}$ or $u_t \leftarrow h_{\lambda_d}^{-1}(\mu_t; c)$ \\
        Apply control to system $x_{t+1} \sim f(x_t, u_t)$ \\
        Shift parameters $\thetatilde_{t+1}$ $\leftarrow$ $\Phi_{\lambda_d}(\theta_t, c)$ \\
        Compute loss $\Jhat_t(\theta_t, \lambda_d; x_t, c)$\\
        Accumulate loss $\ell_d \leftarrow \ell_d + \Jhat_t$
    }
    Update $\lambda_d$ to $\lambda_{d+1}$ with $\nabla_\lambda \ell_d$ using SGD
}
\caption{Training Loop}
\label{alg:training}
\end{algorithm}
\end{minipage}
\vspace{-4ex}
\end{wrapfigure}
To train the \nfmpc\ variants, we following the training procedure in Algorithm \ref{alg:training}, which is carried out training over $D$ episodes.
In each episode $d$, we sample an environment from $\Ccal$ and initial state from $\rho$.
We then perform our rollouts by sampling latent controls $\zhatbold_t^{(1:N)}$, passing them through the normalizing flow to get controls $\uhatbold_t^{(1:N)}$, and then applying them to our approximate dynamics model and cost function to get weights $w_t^{(1:N)}$.
These variables are used to update the latent distribution of the policy to $\theta_t$.
Next, we can either sample a control from the policy or use a control corresponding to the latent mean.
We apply this control to the true system and shift the latent distribution parameters forward with the shift model.
Finally, we compute the loss for the current time step, accumulate the loss to our running sum, and repeat for $T$ time steps.
Once the episode is complete, we update $\lambda_d$ using the gradient of the loss and carry on to the next episode.
Every 100 environments, we test the controller on 10 held out environments and save the current model if it outperforms the previous best on this validation set.
While this introduces a variable number of total episodes used to train the models, we generally find convergence between 2000 and 7000 episodes.

In contrast, to train \flowmppi, we do not actually run an episode.
Instead, we generate the environment and take a gradient step on the initial distribution of the flow conditioned on the environmental information.
We train \flowmppi\ over 10000 randomly generated environments for each task.
We achieved the best performance by initializing the flow with the \flowmppi\ solution, and then refining the flow and learning the shift model jointly as above.
When we condition the NF on additional information, this is simply appended to the current input of each shift and translation network directly.
This conditional information includes start and goal locations for \pngrid\ and \franka\ and start, goal, and obstacle locations for \pnrand\ and \frankaobs.
For \pnranddyn\, we use the current state instead of the initial location, which was found to improve performance for all controllers.
For the obstacles, the conditional information is the Cartesian coordinates of each object of interest stacked together in a vector.

\vspace{-1ex}
\subsection{Additional Planar Navigation Experiments}
\label{app:pn_exp}

\begin{figure*}[t]
    \centering
    \includegraphics[width=\textwidth]{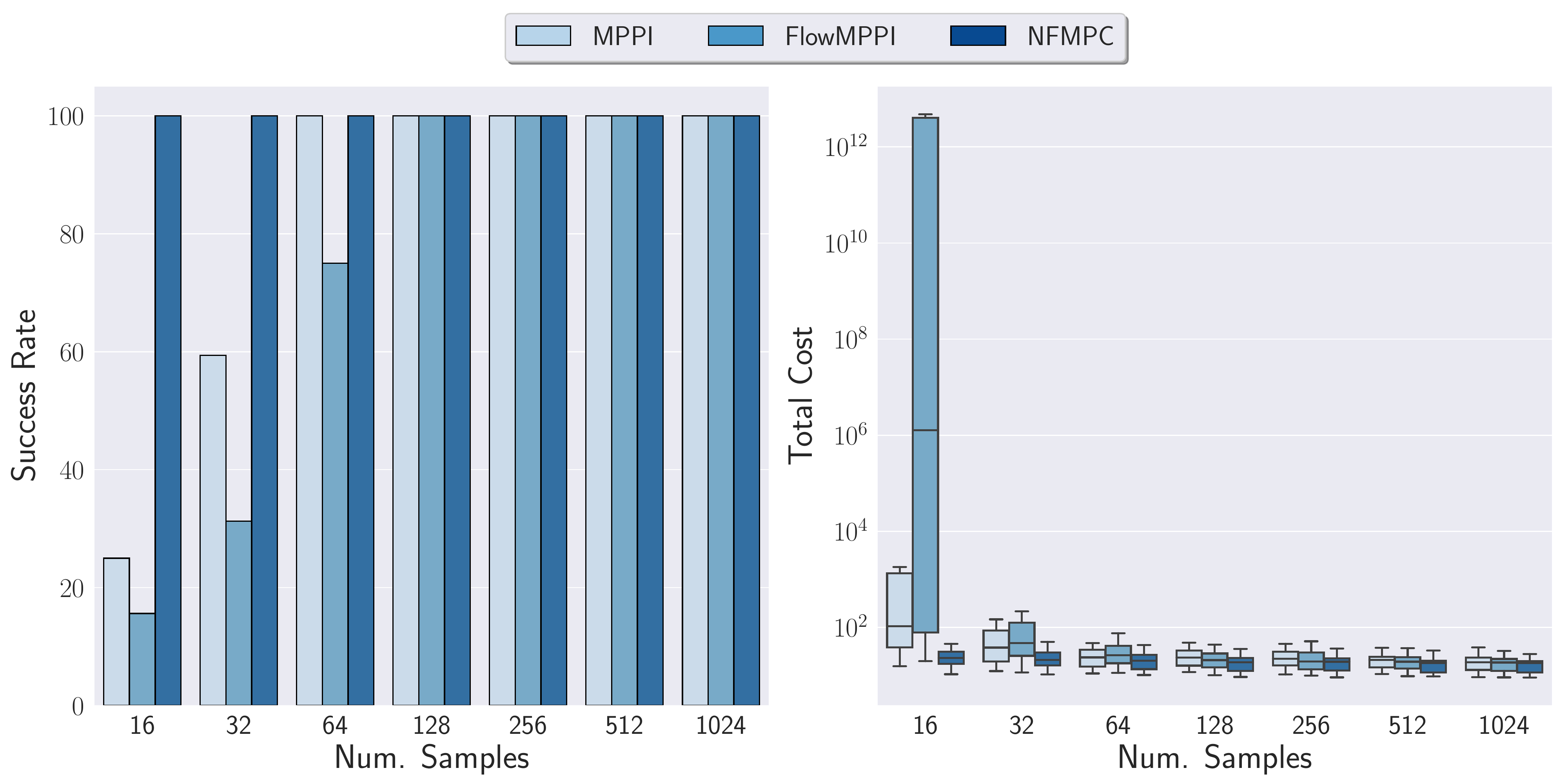}
    \caption{Success rate and cost distribution on the \pngrid\ task across a different number of samples.}
     \label{fig:spngrid_barchart}
     \vspace{-1ex}
\end{figure*}

\begin{figure*}[t]
    \centering
        \begin{minipage}[t]{0.32\textwidth}
            \centering
            \includegraphics[width=\textwidth]{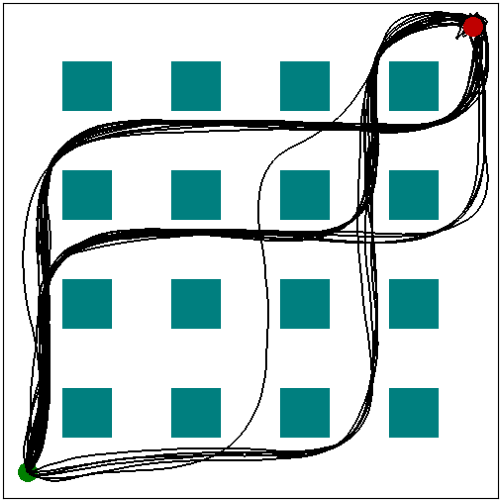}
            NFMPC
        \end{minipage}
        \begin{minipage}[t]{0.32\textwidth}
            \centering
            \includegraphics[width=\textwidth]{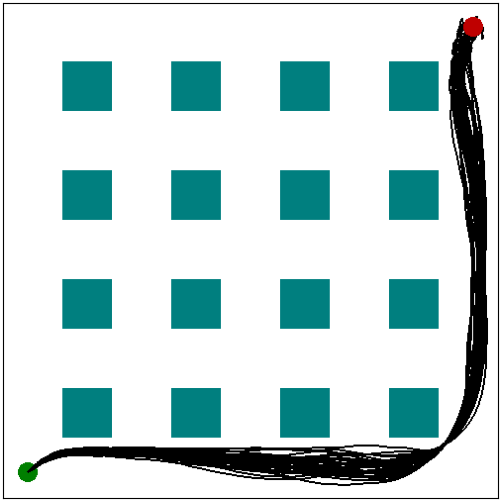}
            FlowMPPI
        \end{minipage}
        \begin{minipage}[t]{0.32\textwidth}
            \centering
            \includegraphics[width=\textwidth]{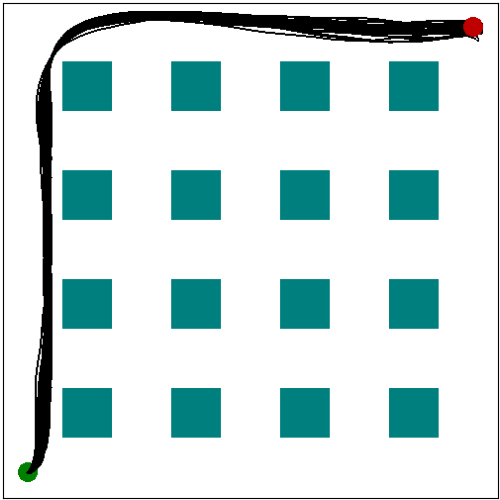}
            MPPI
        \end{minipage}
    \caption{Visualization of trajectories in the \pngrid\ task across multiple random seeds for a fixed environmental layout.}
    \vspace{-4ex}
    \label{fig:spngrid_multitraj}
\end{figure*}

\begin{figure*}[t]
    \begin{minipage}[t]{\textwidth}
        \raisebox{3.8ex}{\rotatebox[origin=t]{90}{NFMPC}}
        \includegraphics[width=0.98\textwidth]{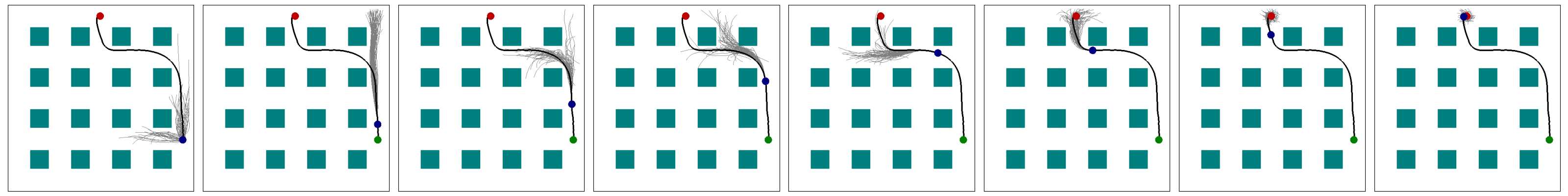}
    \end{minipage}
        \begin{minipage}[t]{\textwidth}
        \raisebox{3.8ex}{\rotatebox[origin=t]{90}{FlowMPPI}}
        \includegraphics[width=0.98\textwidth]{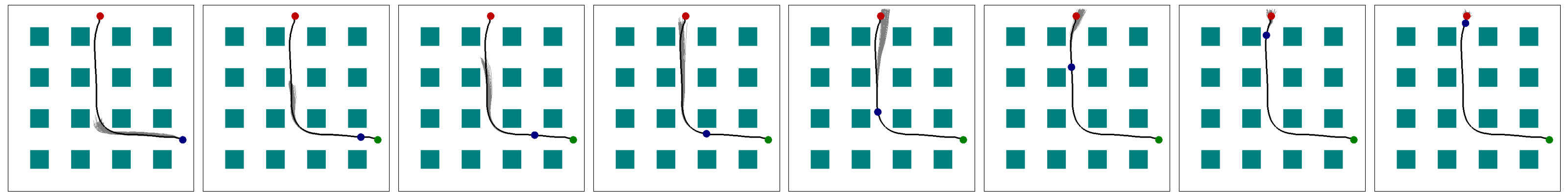}
    \end{minipage}
        \begin{minipage}[t]{\textwidth}
        \raisebox{3.8ex}{\rotatebox[origin=t]{90}{MPPI}}
        \includegraphics[width=0.98\textwidth]{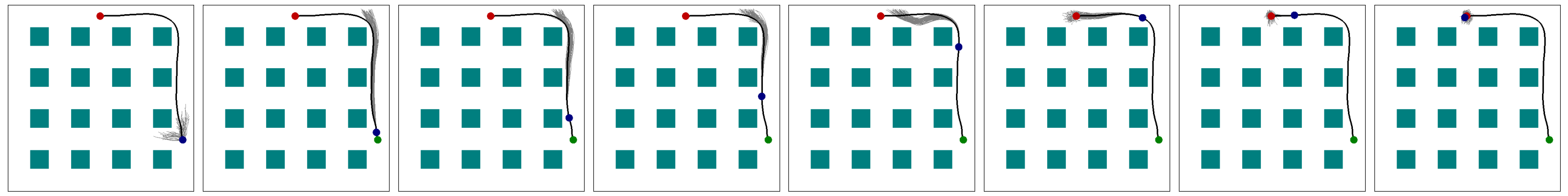}
    \end{minipage}
    \caption{Visualization of a trajectory and top samples from (top) \nfmpc, (middle) \flowmppi, and (bottom) \mppi\ on the \pngrid\ task.}
    \label{fig:spngrid_trajectories}
    \vspace{-4ex}
\end{figure*}

\paragraph{\pngrid.}
We consider a variant of the planar navigation task that involves static obstacles arranged in a grid (\pngrid), described in \Cref{app:exp_details}.
All other task details are the same as in the \pnranddyn\ task, except that we use an unconditional \nfmpc\ controller.
We quantitatively compare all controllers in \Cref{fig:spngrid_barchart} and find that \nfmpc\ consistently matches or outperforms both \mppi\ and \flowmppi\ at each sample quantity in terms of both success rate and average trajectory cost of successful trajectories.
We also find that \nfmpc\ scales more gracefully than \mppi\ as the number of samples is reduced.
In fact, we found that while \flowmppi\ improves over \mppi\ at higher sample counts, it actually performs significantly worse with fewer samples.
This is in contrast to the results on more complex environments considered in the main paper, in which \flowmppi\ generally outperforms \mppi\ at lower sample counts as well.
This is potentially because in the standard implementation of \flowmppi, half of the samples come from the NF and the other half are Gaussian perturbations of the current control-space mean. 
Initially, the samples coming from the NF provide a good initialization for the control distribution mean. 
However, as the latent Gaussian distribution used by the NF is never updated, half of our samples are always coming from this same distribution. 
When the environment or task is more complex, the conditioning information provided to the NF is enough to transform the samples in a useful way.
However, in this simple environment, our hypothesis is that as the robot moves in the environment, these samples may cease to be as useful or informative. 
We then have to rely on the other half of samples coming from Gaussian perturbations of the control-space mean to do most of the work. 
As such, we effectively have half the budget of samples to work with than \mppi\ would, as the samples from the NF potentially do not provide much useful information. 
Meanwhile, because \nfmpc\ is trained recurrently and updates are performed in the latent space, it can better exploit structure in the environment to transform samples.

To better understand what \nfmpc\ is doing differently, we superimpose 32 different trajectories with fixed start and goal positions using each controller in \Cref{fig:spngrid_multitraj}.
We find that both \mppi\ and \flowmppi\ always select the same path through the environment.
Meanwhile, \nfmpc\ is able to discover different paths through the environment, allowing it to better react to the stochastic perturbations that knock it off the current plan and improve performance.
Additionally, we visualize the resulting trajectories and top samples drawn from the distributions for all controllers on one of the validation environments in \Cref{fig:spngrid_trajectories}.
As we would expect, the initial trajectory from \flowmppi\ is better than those of the other two controllers, which basically lay in straight lines in front of the robot.
However, \nfmpc\ is able to discover a slightly faster route to the goal as it proceeds in the environment.
The baseline \mppi\ controller takes a similar, but slightly longer, route to the goal.
Additionally, it takes a longer time to ramp up its velocity compared to the other two controllers, contributing to its sub-optimal performance.

\begin{figure}
    \centering
    \includegraphics[width=\textwidth]{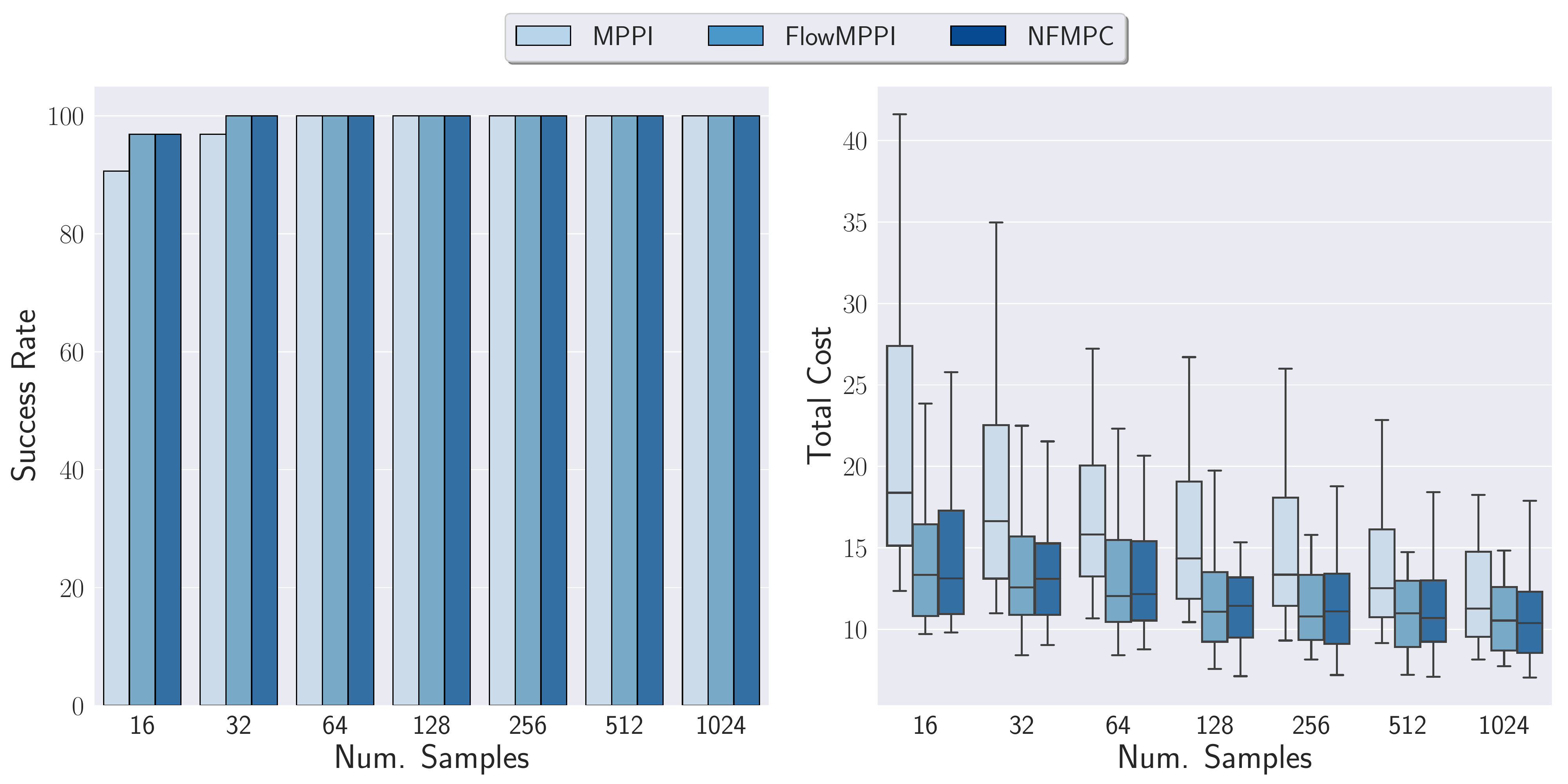}
    \hfill
    \vspace{-3ex}
    \caption{Success rate and cost distribution on the \pnrand\ environment across a different number of samples.}
    \vspace{-2ex}
    \label{fig:spnrand_barchart}
\end{figure}

\begin{figure*}[t]
    \begin{minipage}[t]{\textwidth}
        \raisebox{3.8ex}{\rotatebox[origin=t]{90}{NFMPC}}
        \includegraphics[width=0.98\textwidth]{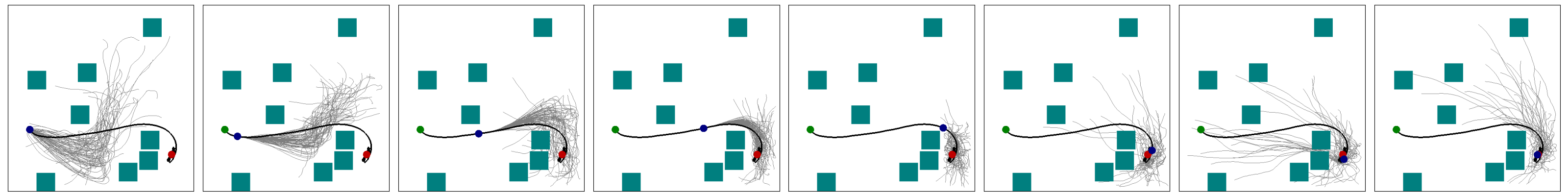}
    \end{minipage}
    \begin{minipage}[t]{\textwidth}
        \raisebox{3.8ex}{\rotatebox[origin=t]{90}{FlowMPPI}}
        \includegraphics[width=0.98\textwidth]{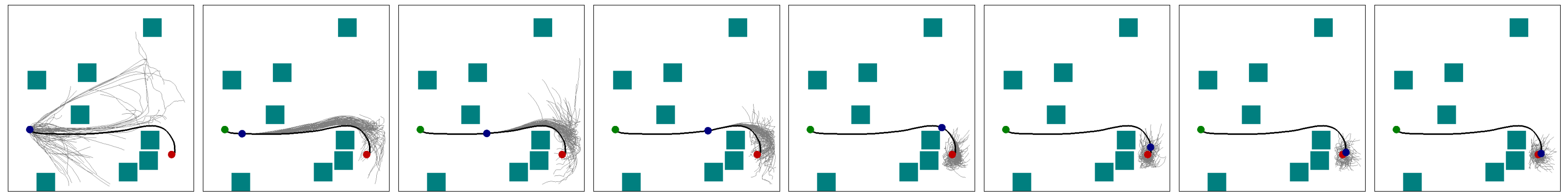}
    \end{minipage}
    \begin{minipage}[t]{\textwidth}
        \raisebox{3.8ex}{\rotatebox[origin=t]{90}{MPPI}}
        \includegraphics[width=0.98\textwidth]{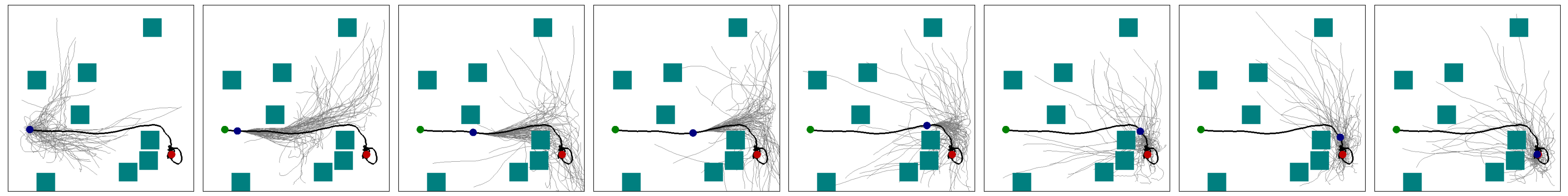}
    \end{minipage}
    \caption{Visualization of a trajectory and top samples from (top) \nfmpc, (middle) \flowmppi, and (bottom) \mppi\ on the \pnrand\ task.}
    \label{fig:spnrand_trajectories}
    \vspace{-2ex}
\end{figure*}

\vspace{-2ex}
\paragraph{\pnrand.}
Next, we consider a variant of \pnranddyn\ in which the eight obstacles are static (\pnrand).
We consider the case where we condition the NF on obstacle locations, initial state, and goal position.
However, it is important to note that we do not condition the shift model, as we found this consistently hurt performance.
In \Cref{fig:spnrand_barchart}, we display the quantitative results and find that \nfmpc\ performed about on par with \flowmppi, which outperformed \mppi.
Unlike on the \pngrid\ task, both \nfmpc\ and \flowmppi\ scale similarly with a reduction of samples and better than \mppi.
This can be partly attributed to the fact that the obstacles are more spaced out and there are more "holes" in the environment than in the grid.
Therefore, it is easier to avoid collisions, possibly contributing to the higher success rates when the controller has access to fewer samples.
Moreover, conditioning on the obstacle locations provides the NF more information, which can be exploited without updating the latent distribution.

We again visualize the trajectories and top samples for all controllers on a validation environment in \Cref{fig:spnrand_trajectories}.
First, we note that the samples in \flowmppi\ appear to be better spread around in the environment to search for good paths towards the goal.
Meanwhile, \nfmpc\ seems to have all top samples directed in one direction.
All models seem to find the same path, with \mppi\ oscillating more near the goal and reaching the goal more slowly than \nfmpc\ and \flowmppi.
The similar performance of \nfmpc\ and \flowmppi\ can be partially attributed to the fact that all we may really need to succeed in this environment is a good initial trajectory which steers us towards the goal.
Whether we refine this trajectory with a learned latent shift model or with Gaussian perturbations in the control space does not appear to make much difference in performance.
However, in \pnranddyn, we demonstrated that there is a significant advantage to our approach, indicating that dynamic environments may be a good application for \nfmpc.

\begin{figure}
    \centering
    \includegraphics[width=\textwidth]{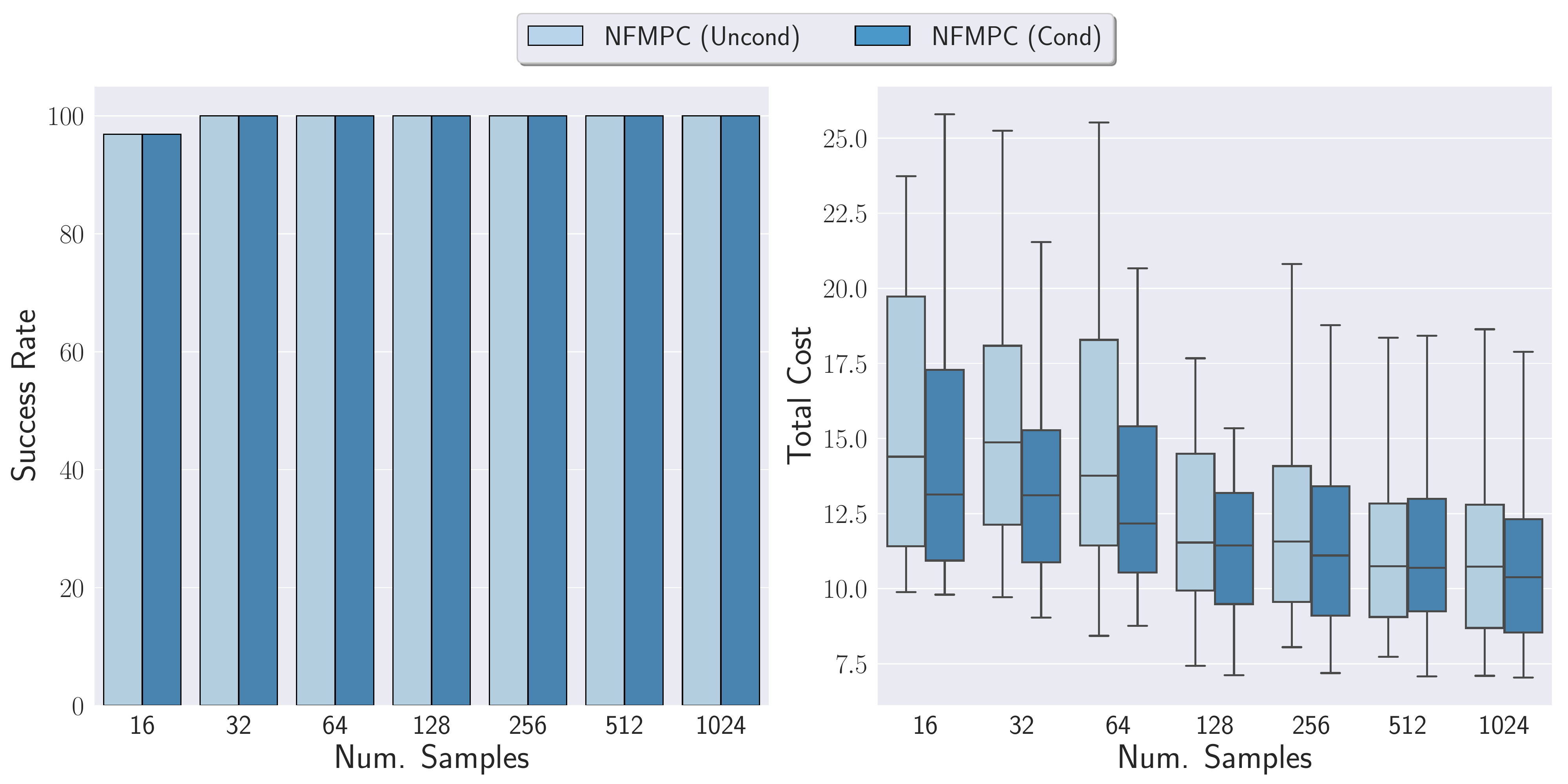}
    \hfill
    \vspace{-3ex}
    \caption{Comparison of unconditional and conditional models on the \pnrand\ environment across a different number of samples.}
    \vspace{-2ex}
    \label{fig:spnrand_barchart_cond}
\end{figure}

Finally, in order to evaluate the benefit of conditioning the NF, we compare the performance of \nfmpc\ with and without conditioning the flow on \pnrand\ in \Cref{fig:spnrand_barchart_cond}.
We find that the conditional model consistently outperforms the unconditional model in terms of median cost, with the gap growing at reduced sample counts.
This may be because the dynamics in \pnrand\ are rather simple, and there may not be much general structure for the unconditional model to exploit since the obstacle locations are entirely random.
Therefore, while the unconditional model does fairly well, conditioning the flow, and not the shift model, seems to enable further gains, even in this simple task.

\begin{figure}
    \centering
    \includegraphics[width=\textwidth]{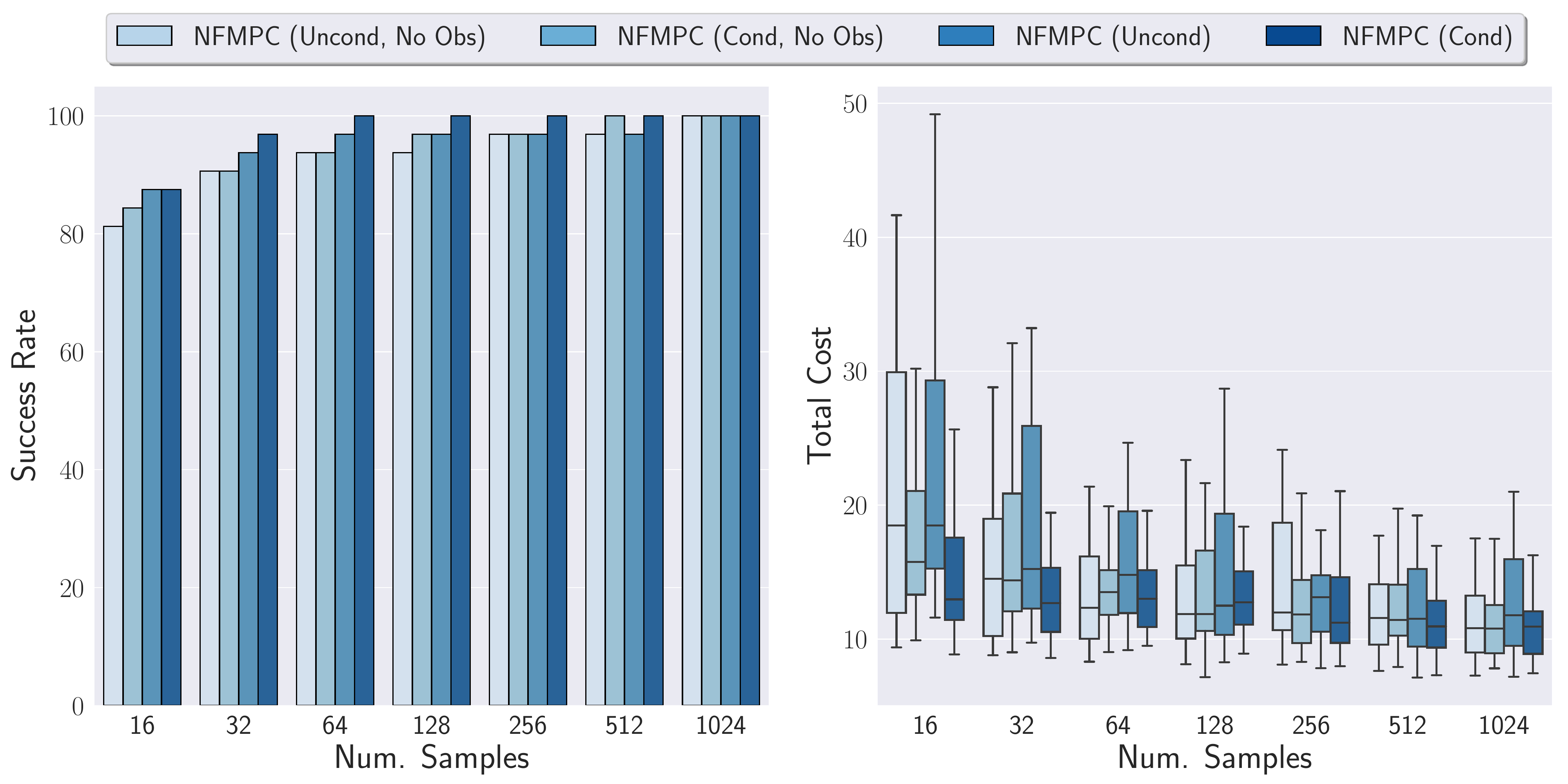}
    \hfill
    \vspace{-3ex}
    \caption{Comparison of unconditional and conditional models on the \pnranddyn\ environment across a different number of samples when trained in an environment with no obstacles (\pnrand) and retrained with obstacles present.}
    \vspace{-2ex}
    \label{fig:spnranddyn_barchart_con_retrain}
\end{figure}

\vspace{-1ex}
\paragraph{\pnranddyn.}
In addition to the experiments in the main paper, we also performed additional ablation studies.
Specifically, we also considered training an unconditional model, as we did before in the \pnrand\ task.
Moreover, we explored transferring controllers trained in the \pnrand\ task to this dynamic version of the environment.
We plot the quantitative results from these ablation studies in \Cref{fig:spnranddyn_barchart_con_retrain}.
Again, we find that conditional models consistently outperform unconditional ones, with the gap in performance more pronounced.
Not only is there a reduction in median cost for unconditional controllers, but it sometimes fails in environments in which its conditional counterpart succeeds.
Transferring the models trained in \pnrand\ works surprisingly well at 1024 samples.
However, at lower sample counts, there is a more pronounced difference in the transferred controllers to those specifically trained in the dynamic environment.
Therefore, it appears that controllers are more efficient at exploring the environments they are trained on, and unsurprisingly, using more samples can partially overcome this gap.


\vspace{-1ex}
\subsection{Additional Franka Experiments}
\vspace{-1ex}
\label{app:franka_exp}

\begin{figure}
    \centering
    \includegraphics[width=\textwidth]{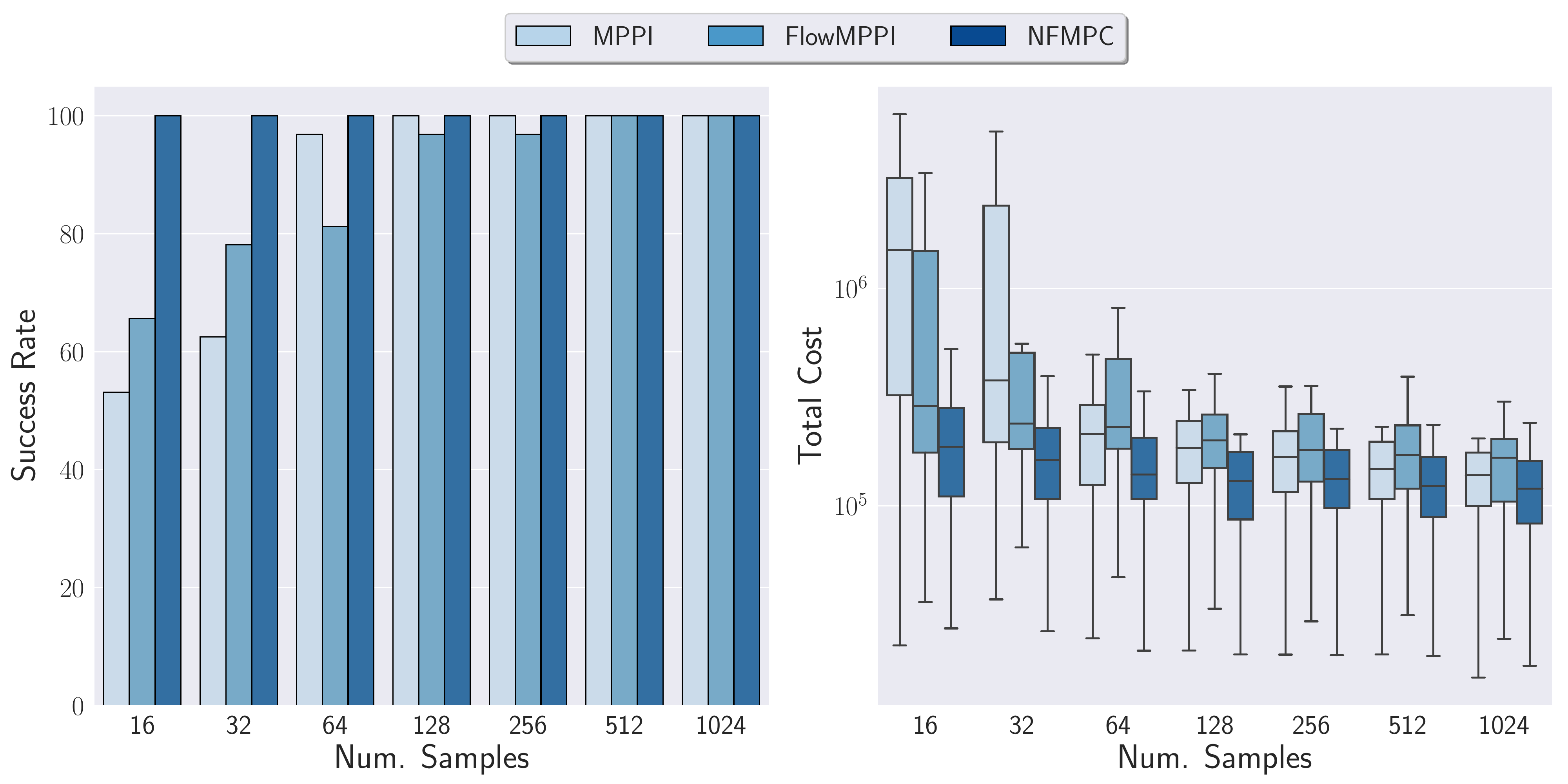}
    \caption{Success rate and cost distribution on the \franka\ environment across a different number of samples.}
    \vspace{-2ex}
    \label{fig:franka_barchart}
\end{figure}

\paragraph{\franka.} We consider a variant of the \frankaobs\ task which involves no obstacles, just a target goal, which we call \franka.
We plot our quantitative results in \Cref{fig:franka_barchart} and find that \nfmpc\ again consistently matches or outperforms \mppi\ and \flowmppi\ at each sample amount.
In fact, \nfmpc\ always achieves a $100\%$ success rate at all sample counts, while both \mppi\ and \flowmppi\ significantly drop in performance at lower amounts of samples.
Moreover, \flowmppi\ sometimes actually performed worse than \mppi, indicating that conditioning on just goal location does not help as much in this more complex scenario.

\begin{figure}
    \centering
    \includegraphics[width=\textwidth]{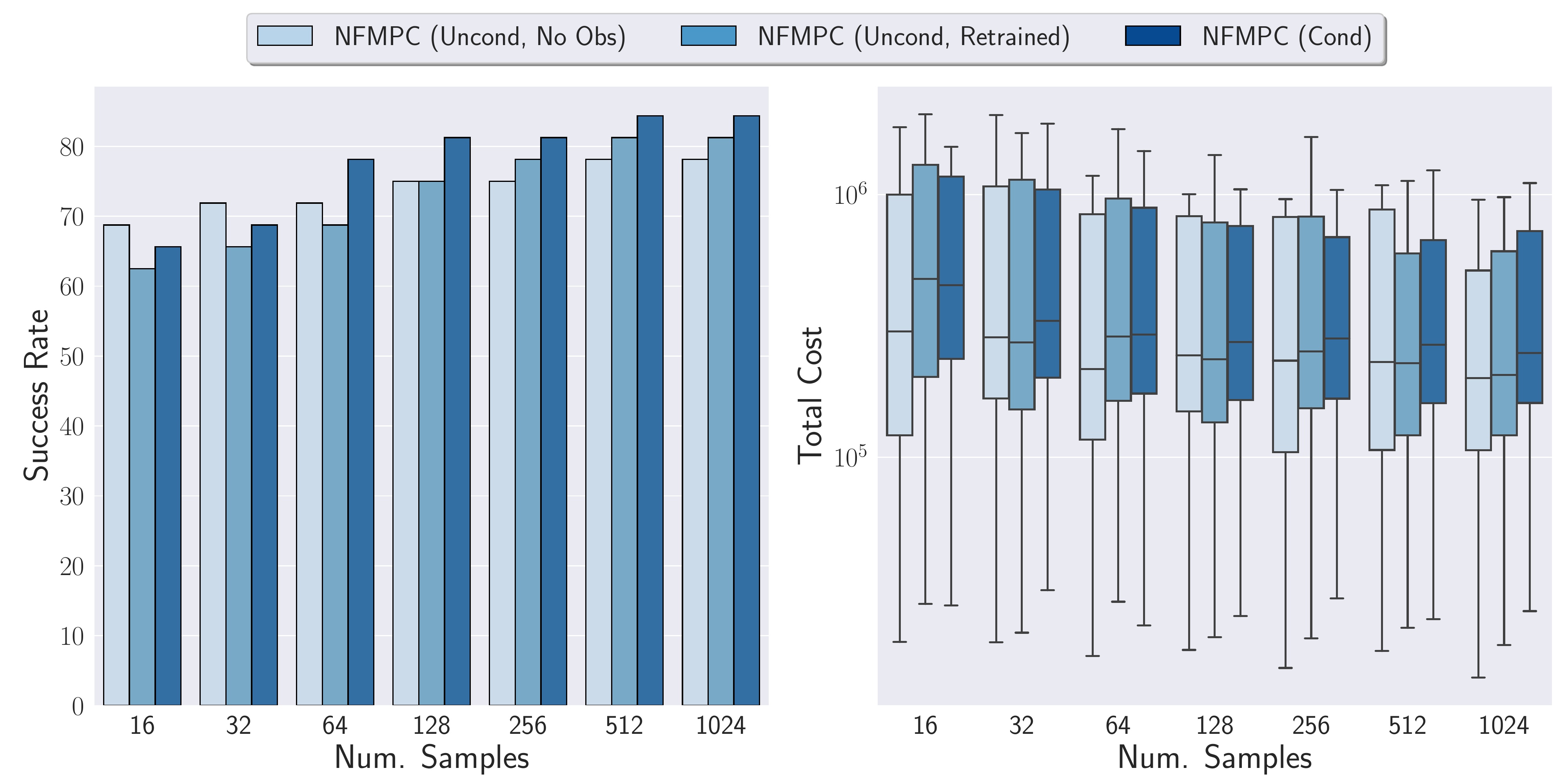}
    \caption{Comparison of an unconditional model trained with and without obstacles to a conditional model in the \frankaobs\ environment across a different number of samples.}
    \vspace{-2ex}
    \label{fig:franka_obs_cond_barchart}
\end{figure}

\vspace{-1ex}
\paragraph{\frankaobs.}
In addition to the results in the main paper, we perform an additional ablation study in which we again compare an unconditional and conditional model on the \frankaobs\ environment in \Cref{fig:franka_obs_cond_barchart}.
We also show the performance of transferring the unconditional model trained on the \franka\ environment without obstacles to an environment which contains the single pole obstacle.
At 1024 samples, the conditional model performs the best in terms of success rate.
Surprisingly, the unconditional model trained without obstacles performs best in terms of median cost and scales better to fewer samples.
Therefore, while the conditional model more often finds a feasible path to the goal, the unconditional model is better able to exploit structure across environments to find lower cost trajectories.
One possible explanation is that because the unconditional model is trained without knowing the specific obstacle locations, it has to be more robust to variation in the environment.
This also shows that transferring the learned distribution to novel environments is possible.
However, since we only have a single static obstacle, it is not clear if these findings would generalize to more complex environments.

\vspace{-2ex}
\paragraph{Breakdown of Trajectory Cost.}
We would like to better understand how the learned controllers improve upon the baseline on the \frankaobs\ task.
As discussed in \Cref{app:exp_details}, the cost function for the Franka tasks is composed of multiple terms.
By inspecting the averages for each term, we hope to gain insight into the learned sampling distributions.
Briefly, we consider a manipulability cost (Manip), which encourages the arm to avoid singular configurations, a self-collision cost (Self), an obstacle collision cost (Obstacle), and a distance-to-goal cost (Goal).
\begin{wrapfigure}{r}{0.5\textwidth}
\centering
\vspace{-0ex}
\includegraphics[width=0.48\textwidth]{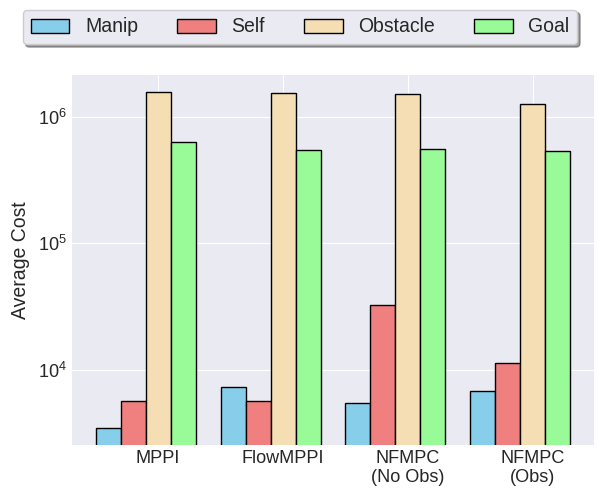}
\caption{Breakdown of different cost terms for each controller executing the \frankaobs\ task.}
\label{fig:cost_breakdown}
\vspace{-2ex}
\end{wrapfigure}
The cost breakdown for these terms in shown in \Cref{fig:cost_breakdown}.
The baseline \mppi\ controller achieves the lowest manipulability cost.
Since this term is not weighted as highly, it makes sense that training the normalizing flow would focus on minimizing terms which were more heavily weighted.
Meanwhile, \flowmppi\ achieves the lowest self-collision cost, which may be due to the fact that it starts off with a better initial plan.
However, \nfmpc\ trained with obstacles (\nfmpco) achieves the lowest obstacle collision and distance-to-goal cost.
One possible explanation for this improvement is that, because we train \nfmpc\ with BPTT, we are potentially able to account for errors which arise due to the inaccurate model.
Additionally, learning the shift model may be an important component to this improvement, which we explore below.
Finally, we note that \nfmpco\ is better than \nfmpc\ trained without obstacles (\nfmpcno) in all cost terms except the manipulability cost, as it is in distribution for the task.

\begin{figure}
    \centering
        \includegraphics[width=\textwidth]{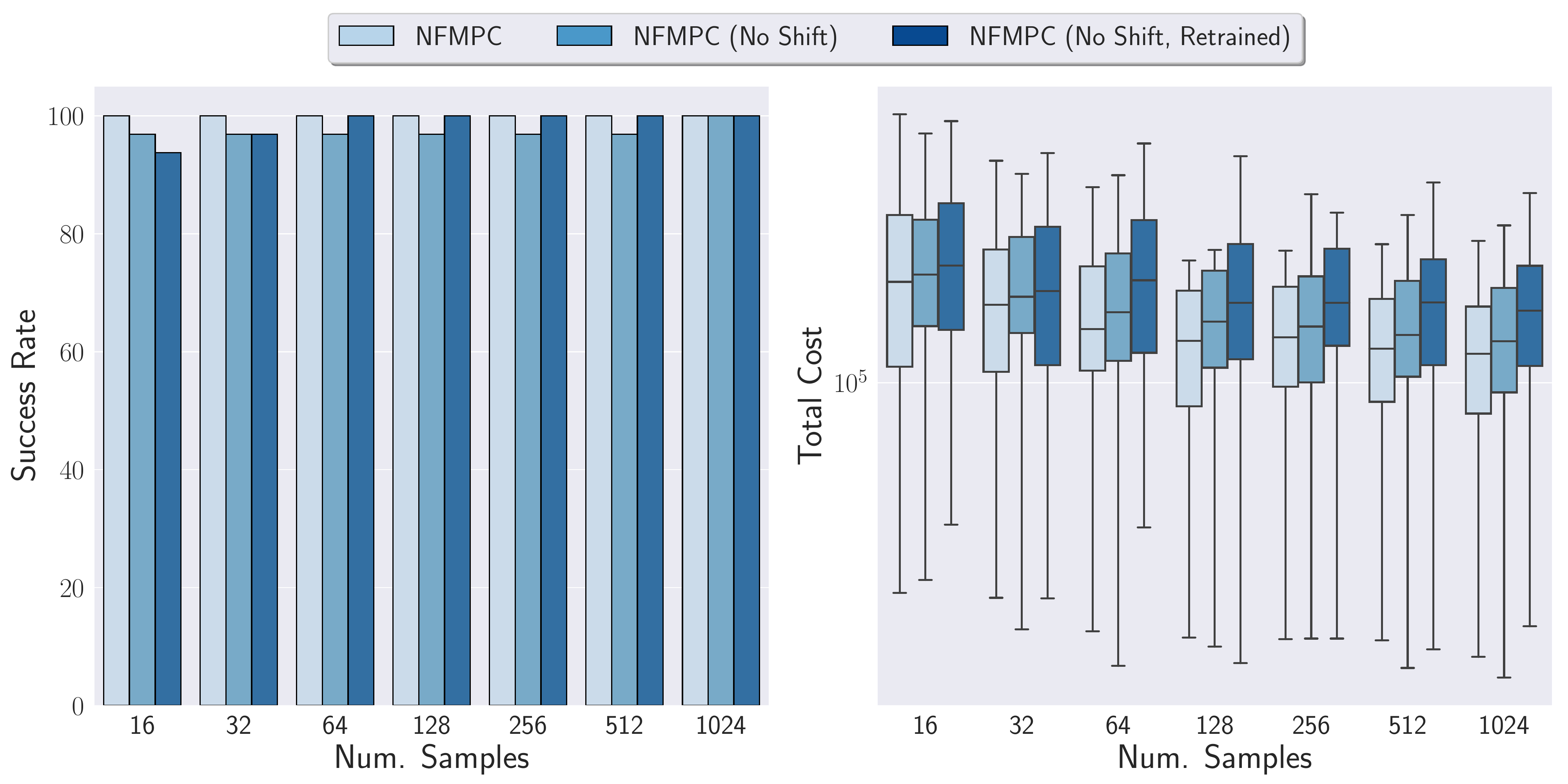}
    \caption{Success rate and cost distribution on the \franka\ task with (\nfmpc) and without (\nfmpcns\ and \textbf{NFMPC (No Shift, Retrained}) the learned shift model.}
    \vspace{-3ex}
    \label{fig:shift_model}
\end{figure}

\vspace{-1ex}
\paragraph{Shift Model Ablation.}
\label{sec:eval_shift}
One of our main contributions which sets us apart from prior work is that we learn a latent shift model and train the controller as a recurrent network.
To determine the impact of this choice, we consider an alternate scenario where the control sequence is instead shifted in control space.
That is, we shift the mean in control space, as in the baseline \mppi\ controller, and then run the flow backwards to infer the corresponding latent mean, which is used to bootstrap the next iteration.
We consider two cases: 1) taking a pre-trained controller and removing the shift model (\nfmpcns) and 2) retraining the controller entirely from scratch without the shift model (\textbf{NFMPC (No Shift, Retrained)}).
In \Cref{fig:shift_model}, we show the results of this ablation study on the \franka\ task.
Removing the shift model strictly hurts performance, even when we retrain the normalizing flow from scratch.
Moreover, retraining the flow actually results in worse performance than simply removing the shift model from the pre-trained controller.
This implies that training the entire controller with BPTT allows it to discover lower-cost trajectories to the goal.

\vspace{-1ex}
\subsection{Performance Overhead of Normalizing Flow}
\vspace{-0.5ex}
\label{app:timing}

We measured the average change in wall clock time across different amounts of samples for \nfmpc\ and \flowmppi\ compared to the baseline \mppi\ implementation on our NVIDIA Titan V GPU. 
For the Planar Robot Navigation and Franka Panda Arm experiments, the average change is $1.61\times$ and $1.91\times$, respectively. 
Moreover, compared to \mppi\ with 1024 samples, the change in wall clock time for \nfmpc\ and \flowmppi\ with 16 samples is approximately $1.01\times$ and $1.14\times$, respectively. 
Therefore, the introduction of the normalizing flow (NF) has a notable impact on wall clock time for both methods. 
However, this overhead is expected and does not prohibit either method’s utility in the real world. 
And there are still clear performance advantages of \nfmpc\ over the baselines in terms of success rate and average trajectory cost for a given sample amount.

Furthermore, the performance of \mppi\ reported in the paper and the above timing comparisons are when the optimization is run for a single iteration per time step. 
However, the performance of \mppi\ generally improves with an increased number of iterations at the cost of an increased runtime. 
For instance, we compare the performance of \mppi\ run for 3 iterations per time step with \nfmpc\ run for a single iteration, both using 1024 samples. 
In the \franka\ task, we can reduce the average trajectory cost of MPPI to be only $12\%$ worse than \nfmpc, rather than the previous $19\%$. 
When we introduce obstacles for the \frankaobs task, \mppi\ with 3 iterations can actually match the success rate of \nfmpc, albeit with a worse average trajectory cost. 
However, in this case, \nfmpc\ actually results in an average reduction of wall clock time by $24.8\%$. 
Similarly, for the \pngrid\ task, \mppi\ with 3 iterations reduces the average trajectory cost to be only $10\%$ worse than \nfmpc, rather than the previous $40\%$. 
However, this again comes at the cost of increased runtime, as \nfmpc\ reduces the average wall clock time by $26.8\%$. 
Therefore, \nfmpc\ allows us to surpass the performance of \mppi\ run with more iterations while reducing the required runtime.

Additionally, \nfmpc\ has substantial runtime benefits over \flowmppi\ due to its improved scaling. 
For the \pngrid task, \nfmpc\ with 128 samples outperforms \flowmppi\ with 1024 samples while reducing average runtime by $22\%$. 
Similarly, for the \franka task, \nfmpc\ with 64 samples outperforms \flowmppi\ with 1024 samples while reducing average runtime by $43\%$. 
And for both Franka tasks (\franka\ and \frankaobs), \nfmpc\ with 16 samples outperforms \flowmppi\ with 128 samples while reducing average runtime by $8\%$. 
As such, there are clear runtime benefits for \nfmpc\ over both \flowmppi\ and \mppi\ run for additional iterations. 
Finally, it is also important to note that we did not perform a hyperparameter sweep on the NF, and it may be possible to significantly reduce the size of the network while retaining performance benefits.

\vspace{-1ex}
\subsection{Proof for DMD Update of Latent Parameters}
\label{app:proof_latent_update}
\begin{theorem}
Consider the optimization problem
\begin{equation}
    \theta_t = \argmin_{\theta \in \Theta} \inner{\gamma_t g_t}{\theta} + D_{KL}(\pi_{\theta, \lambda} || \pi_{\thetatilde_t, \lambda})
    \label{eq:appendix_dmd}
\end{equation}
where we define
\begin{equation}
\pi_{\theta, \lambda}(\uhatbold | c) = p_\theta(h_{\lambda}(\uhatbold; c)) \Det{ \pderiv{h_{\lambda}}{\uhatbold} }
\end{equation}
and $h_\lambda$ is an invertible, deterministic transformation, $p_\theta$ a Gaussian factorized as in \Cref{eq:factorized_gaussian}, and $\theta$ is the natural parameters of the Gaussian.
Then the corresponding update to the mean of the latent Gaussian is given by:
\begin{equation}
\mubold_t = (1-\gamma_t^\mu) \mutildebold_t + \gamma_t^\mu \frac{\Expect{\pi_{\thetatilde, \lambda}, \fhat}{e^{ -\frac{1}{\beta} C(\xhatbold_t, \uhatbold_t) } h_\lambda(\uhatbold_t; c)}}{\Expect{\pi_{\thetatilde, \lambda}, \fhat}{e^{ -\frac{1}{\beta} C(\xhatbold_t, \uhatbold_t)} }}
= (1-\gamma_t^\mu) \mutildebold_t + \gamma_t^\mu \frac{\Expect{p_{\thetatilde}, \fhat}{e^{ -\frac{1}{\beta} C(\xhatbold_t, h_\lambda^{-1}(\zhatbold_t; c)) } \zhatbold_t}}{\Expect{p_{\thetatilde}, \fhat}{e^{ -\frac{1}{\beta} C(\xhatbold_t, h_\lambda^{-1}(\zhatbold_t; c))} }}
\end{equation}
\end{theorem}

\begin{proof}
First, we note that
\begin{equation}
g_t = \nabla \Jhat(\thetatilde; x_t) = - \frac{
    \Expect{\pi_{\thetatilde, \lambda}, \fhat}{e^{-\frac{1}{\beta}C(\xhatbold_t, \uhatbold_t)} \nabla_{\thetatilde} \log \pi_{\thetatilde, \lambda}(\uhatbold_t)}
}{
    \Expect{\pi_{\thetatilde, \lambda}, \fhat}{e^{-\frac{1}{\beta}C(\xhatbold_t, \uhatbold_t)}}
},
\end{equation}
and that
\begin{equation}
\log \pi_{\thetatilde, \lambda}(\uhatbold | c) = \log p_{\thetatilde} (h_{\lambda}(\uhatbold; c)) + \log \Big( \Det{\pderiv{h_{\lambda}}{\uhatbold}} \Big).
\label{eq:nf_log_likelihood_appendix}
\end{equation}
Therefore, when computing the gradient of \Cref{eq:nf_log_likelihood_appendix} with respect to $\thetatilde$, we can drop the log-determinant term as it does not depend on $\thetatilde$. 
As such, we are left with the gradient of the latent Gaussian with respect to its natural parameters, or $\nabla_{\thetatilde} \log \pi_{\thetatilde, \lambda}(\uhatbold | c) = \nabla_{\thetatilde} \log p_{\thetatilde} (h_{\lambda}(\uhatbold; c))$.
Taking this gradient of the log-likelihood term with respect to the natural parameters, we have:
\begin{equation}
g_t = - \frac{
    \Expect{\pi_{\thetatilde, \lambda}, \fhat}{e^{-\frac{1}{\beta}C(\xhatbold_t, \uhatbold_t)} \big( \phi(h_{\lambda}(\uhatbold_t; c)) - \phitilde_t \big)}
}{
    \Expect{\pi_{\thetatilde, \lambda}, \fhat}{e^{-\frac{1}{\beta}C(\xhatbold_t, \uhatbold_t)}}
},
\end{equation}
where $\phitilde_t$ is the expectation parameter corresponding to natural parameter $\thetatilde_t$ and $\phi(\cdot)$ is the sufficient statistics of the latent distribution.
We can rewrite the expectations in terms of $p_{\thetatilde}(\cdot)$:
\begin{equation}
g_t = - \frac{
    \Expect{p_{\thetatilde}, \fhat}{e^{-\frac{1}{\beta}C(\xhatbold_t, h_{\lambda}^{-1}(\zhatbold_t; c)))} \big( \phi(\zhatbold_t) - \phitilde_t \big)}
}{
    \Expect{p_{\thetatilde}, \fhat}{e^{-\frac{1}{\beta}C(\xhatbold_t, h_{\lambda}^{-1}(\zhatbold_t; c))}}
},
\end{equation}
Next, looking at the KL divergence term, we can write:
\begin{equation}
\begin{aligned}
D_{KL}(\pi_{\theta, \lambda} || \pi_{\thetatilde_t, \lambda}) 
&= \Expectt{\pi_{\theta, \lambda}}{\log \frac{\pi_{\theta, \lambda}(\uhatbold)}{\pi_{\thetatilde_t, \lambda}(\uhatbold)}} \\
&= \Expectt{\pi_{\theta, \lambda}}{\log \frac{p_\theta(h_{\lambda}(\uhatbold; c)) }{p_{\thetatilde}(h_{\lambda}(\uhatbold; c)) } 
\cancel{\frac{\Det{\pderiv{h_{\lambda}}{\uhatbold}} }{\Det{\pderiv{h_{\lambda}}{\uhatbold}} }}
} \\
&= \Expectt{p_\theta}{\log \frac{p_\theta(\zhatbold) }{p_{\thetatilde}(\zhatbold) } } \\
&= D_{KL} (p_{\theta} || p_{\thetatilde_t}),
\end{aligned}
\end{equation}
Then, using Proposition 1 from \citet{wagener2019online}, we can write the update rule as
\begin{equation}
\phi_t = (1 - \gamma_t) \phitilde_t + \gamma_t
\frac{
    \Expect{p_{\thetatilde}, \fhat}{e^{-\frac{1}{\beta}C(\xhatbold_t, h_{\lambda}^{-1}(\zhatbold_t; c)))} \phi(\zhatbold_t)}
}{
    \Expect{p_{\thetatilde}, \fhat}{e^{-\frac{1}{\beta}C(\xhatbold_t, h_{\lambda}^{-1}(\zhatbold_t; c))}}
}
\end{equation}
And when the sufficient statistic is $\phi(\zhatbold_t) = (\zhatbold_t, \zhatbold_t \zhatbold_t^T)$, then we arrive at the usual MPPI update for the mean, but now defined in terms of the latent samples:
\begin{equation}
\mubold_t
= (1-\gamma_t^\mu) \mutildebold_t + \gamma_t^\mu \frac{\Expect{p_{\thetatilde}, \fhat}{e^{ -\frac{1}{\beta} C(\xhatbold_t, h_\lambda^{-1}(\zhatbold_t; c)) } \zhatbold_t}}{\Expect{p_{\thetatilde}, \fhat}{e^{ -\frac{1}{\beta} C(\xhatbold_t, h_\lambda^{-1}(\zhatbold_t; c))} }}
\end{equation}
which we can equivalently rewrite in terms of $\pi_{\thetatilde, \lambda}$ as:
\begin{equation}
\mubold_t = (1-\gamma_t^\mu) \mutildebold_t + \gamma_t^\mu \frac{\Expect{\pi_{\thetatilde, \lambda}, \fhat}{e^{ -\frac{1}{\beta} C(\xhatbold_t, \uhatbold_t) } h_\lambda(\uhatbold_t; c)}}{\Expect{\pi_{\thetatilde, \lambda}, \fhat}{e^{ -\frac{1}{\beta} C(\xhatbold_t, \uhatbold_t)} }}
\end{equation}
\end{proof}

\subsection{Proof of Approximate Gradient Through MPPI Update}
\label{app:approx_gradient_proof}

\begin{theorem}
Let the MPPI update of the latent mean be given by
\begin{equation}
\mubold_{t}(\lambda) = (1-\gamma_t^\mu) \mutildebold_{t}(\lambda) + \gamma_t^\mu \Delta \mubold_{t},\qquad
\Delta \mubold_{t} = 
\frac{\Expect{\pi_{\thetatilde_{t}(\lambda), \lambda}, \fhat}{e^{ -\frac{1}{\beta} C(\xhatbold_t, \uhatbold_t) } h_{\lambda}(\uhatbold_{t}; c) }}{\Expect{\pi_{\thetatilde_{t}(\lambda), \lambda}, \fhat}{e^{ -\frac{1}{\beta} C(\xhatbold_t, \uhatbold_t)} }}.
\end{equation}
Then the gradient of $\Delta \mubold_{t}$ with respect to $\lambda$ can be approximated as
\begin{equation}
\pderiv{\Delta \mubold_t}{\lambda} \approx M_1  - M_2 M_3 
\end{equation}
where
\begin{equation}
\begin{gathered}
M_1 = \sum_{i=1}^N w_i \Big[ \nabla_\lambda h_\lambda(\uhatbold_t^{(i)}; c) + h_\lambda(\uhatbold_t^{(i)}; c) \nabla_\lambda \log \pi_{\thetatilde(\lambda),\lambda}(\uhatbold_t^{(i)} | c) \Big],\\
M_2 = \sum_{i=1}^N w_i h_\lambda(\uhatbold_t^{(i)}; c),\quad 
M_3 = \sum_{i=1}^N w_i \nabla_\lambda \log \pi_{\thetatilde(\lambda),\lambda}(\uhatbold_t^{(i)} | c).
\end{gathered}
\end{equation}
\end{theorem}
\begin{proof}
First, we rewrite $\Delta \mubold_t = \frac{N(\lambda)}{D(\lambda)}$, where
\begin{equation}
N(\lambda) = \Expect{\pi_{\thetatilde_{t}(\lambda), \lambda}, \fhat}{e^{ -\frac{1}{\beta} C(\xhatbold_t, \uhatbold_t) } h_{\lambda}(\uhatbold_{t}; c) }\qquad
D(\lambda) = \Expect{\pi_{\thetatilde_{t}(\lambda), \lambda}, \fhat}{e^{ -\frac{1}{\beta} C(\xhatbold_t, \uhatbold_t)} }.
\end{equation}
Then by the quotient rule of calculus, we have
\begin{equation}
\pderiv{\Delta \mubold_t}{\lambda} = \frac{\nabla_\lambda N(\lambda)}{D(\lambda)} - \frac{N(\lambda)}{D(\lambda)} \frac{\nabla_\lambda D(\lambda)}{D(\lambda)}.
\label{eq:approx_grad_symbolic}
\end{equation}
We can compute each of these individual terms using the likelihood ratio gradients
\begin{equation}
\nabla_\lambda N(\lambda) = \Expect{\pi_{\thetatilde_{t}(\lambda), \lambda}, \fhat}{
e^{ -\frac{1}{\beta} C(\xhatbold_t, \uhatbold_t) } 
\Par{
    \nabla_\lambda  h_{\lambda}(\uhatbold_{t}; c) + 
    h_{\lambda}(\uhatbold_{t}; c) \nabla_\lambda \log \pi_{\thetatilde(\lambda),\lambda}(\uhatbold_t^{(i)} | c)
}
}
\end{equation}
\begin{equation}
\nabla_\lambda D(\lambda)  = \Expect{\pi_{\thetatilde_{t}(\lambda), \lambda}, \fhat}{
e^{ -\frac{1}{\beta} C(\xhatbold_t, \uhatbold_t) } 
\nabla_\lambda \log \pi_{\thetatilde(\lambda),\lambda}(\uhatbold_t^{(i)} | c) 
}
\end{equation}
Since each of the terms that make up our gradient in \Cref{eq:approx_grad_symbolic} are divided by $D(\lambda)$, when we approximate them with Monte Carlo sampling, we can actually write them in terms of the same weights used by MPPI in the forward pass:
\begin{subequations}
\begin{align}
\frac{\nabla_\lambda N(\lambda)}{D(\lambda)} &= \sum_{i=1}^N w_i \Big[ \nabla_\lambda h_\lambda(\uhatbold_t^{(i)}; c) + h_\lambda(\uhatbold_t^{(i)}; c) \nabla_\lambda \log \pi_{\thetatilde(\lambda),\lambda}(\uhatbold_t^{(i)} | c) \Big] = M_1\\
\frac{N(\lambda)}{D(\lambda)} &= \sum_{i=1}^N w_i h_\lambda(\uhatbold_t^{(i)}; c) = M_2\\
\frac{\nabla_\lambda D(\lambda)}{D(\lambda)} &= \sum_{i=1}^N w_i \nabla_\lambda \log \pi_{\thetatilde(\lambda),\lambda}(\uhatbold_t^{(i)} | c) = M_3
\end{align}
\end{subequations}
\end{proof}

\vspace{-4ex}
\subsection{Sigmoid Flow Layer} 
\vspace{-1ex}
\label{app:sigmoid}
We wish to use a sigmoid layer in our normalizing flow to constrain our control sample $\uhatbold$ such that each control along the horizon is between $\barbelow{u}$ and $\bar{u}$.
Since the sigmoid function is invertible, if we append a sigmoid layer at the end of our flow, we have that
\begin{equation}
    \uhatbold = w\sigma(\yhatbold_{K-1}) + b\quad
    \Longleftrightarrow\quad
    \yhatbold_{K-1} = \sigma^{-1} \Big(\frac{\uhatbold - b}{w} \Big) = \log \Big( \frac{\uhatbold - b}{w - \uhatbold + b} \Big),
\end{equation}
where $w = \bar{u} - \barbelow{u}$ and $b = \barbelow{u}$, the sigmoid and logit functions are applied element-wise, and the scaling and translation are broadcasted to each element of the vector $\uhatbold$.
The derivative of the forward transformation is given by:
\begin{equation}
\begin{aligned}
\pderiv{}{x}\Big( w\sigma(x)+b \Big) &= w \sigma(x) (1 - \sigma(x)) \\
&= \frac{w}{1+e^{-x}} \Big( 1 - \frac{1}{1+e^{-x}} \Big) \\
&= \frac{w}{(1+e^{-x})(1+e^x)}.
\end{aligned}
\end{equation}
Since the sigmoid is applied element-wise, it has a diagonal Jacobian, the log-determinant of which is simply the sum of the log of its diagonal terms:
\begin{equation}
\log \Det{\pderiv{\uhatbold}{\yhatbold_{K-1}}} = 
\sum_{i=1}^{MH} \log(w) - \log(1+e^{-\uhat_i}) - \log(1+e^{\uhat_i}),
\end{equation}
where we can implement the last two terms with the Softplus activation function.
In the reverse direction, we have that:
\begin{equation}
\begin{aligned}
\pderiv{}{x} \Bigg( \sigma^{-1} \Big(\frac{x - b}{w} \Big) \Bigg)
&= \frac{1}{x-b} + \frac{1}{w-x-b} 
\end{aligned}
\end{equation}
Therefore, the log-determinant is given by:
\begin{equation}
\log \Det{\pderiv{\yhatbold_{K-1}}{\uhatbold}} = 
- \sum_{i=1}^{MH} \Big(\log(\uhat_i-b) + \log(w - \uhat_i -b) \Big),
\end{equation}
As such, computing the inverse and the log-determinant terms of the Jacobian is efficient and fast in both directions, adding minimal overhead to the flow.

\end{document}